\documentclass[sigconf]{acmart}

\AtBeginDocument{%
  \providecommand\BibTeX{{%
    \normalfont B\kern-0.5em{\scshape i\kern-0.25em b}\kern-0.8em\TeX}}}

\copyrightyear{2023}
\acmYear{2023}
\setcopyright{acmlicensed}
\acmConference[WWW '23]{Proceedings of the ACM Web Conference 2023}{May 1--5, 2023}{Austin, TX, USA}
\acmBooktitle{Proceedings of the ACM Web Conference 2023 (WWW '23), May 1--5, 2023, Austin, TX, USA}
\acmPrice{15.00}
\acmDOI{10.1145/3543507.3583362}
\acmISBN{978-1-4503-9416-1/23/04}

\usepackage{times}  
\usepackage{helvet}  
\usepackage{courier}  
\usepackage{graphicx} 
\usepackage{natbib}  
\usepackage{caption} 
\usepackage{makecell}
\usepackage{amsthm}
\usepackage{subfigure}
\usepackage{xcolor}
\newtheorem{theorem}{Theorem}
\newtheorem{lemma}{Lemma}

\usepackage{multirow}
\usepackage{stfloats}
\usepackage{hyperref}
\DeclareMathOperator*{\argmax}{arg\,max}
\DeclareMathOperator*{\argmin}{arg\,min}
\newtheorem{corollary}{Corollary}
\newtheorem{definition}{Definition}
\newtheorem{proposition}{Proposition}
\usepackage{bm}

\usepackage{enumitem}
\usepackage{pifont}
\newcommand{\cmark}{\checkmark}
\newcommand{\xmark}{\ding{53}}

\theoremstyle{remark}
\newtheorem{strategy}{Strategy}

\usepackage[ruled,linesnumbered]{algorithm2e}

\usepackage{listings}

\begin{document}

\title{Clustered Embedding Learning for Recommender Systems} 

\author{Yizhou Chen}\authornote{Both authors contributed equally to the paper.}

\author{Guangda Huzhang}\authornotemark[1]
\email{ {yizhou.chen, guangda.huzhang} }
\email{@shopee.com}
\affiliation{
  \institution{Shopee Pte Ltd., Singapore}
\country{}
}

\author{Anxiang Zeng}
\email{zeng0118@ntu.edu.sg}
\affiliation{
  \institution{SCSE, Nanyang Technological University, Singapore}
\country{}
}

\author{Qingtao Yu, Hui Sun}
\author{Heng-Yi Li, Jingyi Li}
\email{ {qingtao.yu, halsey.sun, hengyi.li, lijingyi} }
\email{@shopee.com}

\affiliation{
  \institution{Shopee Pte Ltd., Singapore}
\country{}
}

\author{Yabo Ni}
\email{yabo001@e.ntu.edu.sg}
\affiliation{
  \institution{SCSE, Nanyang Technological University, Singapore}
  \country{}
}

\author{Han Yu}
\email{han.yu@ntu.edu.sg}
\affiliation{
  \institution{SCSE, Nanyang Technological University, Singapore}
\country{}
}

\author{Zhiming Zhou}
\authornote{Corresponding author.}

\email{zhouzhiming@mail.shufe.edu.cn}
\affiliation{
  \institution{Shanghai University of Finance and Economics, Shanghai, China}
\country{\hspace{1pt}} 
}

\renewcommand{\shortauthors}{Chen, Huzhang and Zeng, et al.}

\begin{abstract}
In recent years, recommender systems have advanced rapidly, where embedding learning for users and items plays a critical role. A standard method learns a unique embedding vector for each user and item. However, such a method has two important limitations in real-world applications: 1) it is hard to learn embeddings that generalize well for users and items with rare interactions; 
and 2) it may incur unbearably high memory costs when the number of users and items scales up. Existing approaches either can only address one of the limitations or have flawed overall performances. 
In this paper, we propose Clustered Embedding Learning (CEL) as an integrated solution to these two problems. CEL is a plug-and-play embedding learning framework that can be combined with any differentiable feature interaction model. It is capable of achieving improved performance, especially for cold users and items, with reduced memory cost. CEL enables automatic and dynamic clustering of users and items in a top-down fashion, where clustered entities jointly learn a shared embedding. The accelerated version of CEL has an optimal time complexity, which supports efficient online updates. 
Theoretically, we prove the identifiability and the existence of a unique optimal number of clusters for CEL in the context of nonnegative matrix factorization. Empirically, we validate the effectiveness of CEL on three public datasets and one business dataset, showing its consistently superior performance against current state-of-the-art methods. In particular, when incorporating CEL into the business model, it brings an improvement of $+0.6\%$ in AUC, which translates into a significant revenue gain; meanwhile, the size of the embedding table gets $2650$ times smaller.\footnote{The code is avaliable at: \url{https://doi.org/10.5281/zenodo.7620448}.}
\end{abstract}



\maketitle

\section{Introduction}

Recommender systems, which support many real-world applications, have become increasingly accurate in recent years. 
Important improvements are brought about by the use of deep learning \cite{cheng2016wide,guo2017deepfm} and the leverage of user historical data~\cite{din,dien}. In these systems, embedding learning (i.e., learning vector representations) for users and items plays a critical role.

The standard method, \emph{full embedding}, learns a unique vector representation for each user and item. However, this approach faces important limitations in practice. 
Firstly, for cold users and items that have insufficient historical interactions (which commonly exist in practice), it is hard for them to learn an embedding that generalizes well on their own. 
Secondly, web-scale recommender systems usually involve an enormous number of users and items; the memory cost incurred by full embedding can be prohibitively high: 1 billion users, each assigned a $64$ dimensional embedding in $32$-bit floating point, require $238$ GB of memory.

To mitigate the problem of cold users and items, a common idea is to impose regularization, e.g., forcing all embeddings to follow a given clustering structure \cite{eTREE,jnkm}. However, using a fixed predefined clustering structure lacks flexibility; the given clustering structure can hardly be optimal. Besides, existing methods \cite{eTREE,jnkm} require training in full embedding, still suffering from high memory costs. 

To mitigate the problem of high memory cost, hashing \cite{modulo,binaryhash,hashgnn,prehash} is typically adopted to map entities to a fixed-size embedding table. This can directly tackle the problem of memory cost. However, forcing a hashing that is essentially random \cite{modulo,binaryhash} leads to unsatisfactory performance; the performances of current learning-to-hash methods \cite{hashgnn,prehash} also have a large space for improvements. Besides, pretraining a hash function \cite{hashgnn} or binding each hash bucket with a warm user \cite{prehash} also makes these methods lack flexibility.

In chasing a more flexible and integrated solution to these two problems, while at the same time improving the performance, we propose in this paper Clustered Embedding Learning (CEL). CEL achieves dynamic and automatic clustering of users and items, which gives it the potential to have superior performance. In CEL, entities within the same cluster may use and jointly learn a shared embedding; thus it can significantly reduce the memory cost. 

CEL performs clustering in a top-down (i.e., divisive) fashion, starting with a single cluster that includes all users (or items) and recursively splitting a cluster into two. Given that we want to share information among similar entities, such divisive clustering is a natural choice: sharing information of all entities at the beginning, then seeking to refine the clustering when possible. 

Though being a natural thought, implementing divisive clustering for embedding learning is challenging. In the circumstance of \textit{shared embedding}, it appears to have no standard to differentiate entities within the same cluster. To tackle this problem, we propose to refer to their gradients with respect to the shared embedding, since the differences in the gradients indicate their different desired optimization directions of the embedding. Then naturally, we propose to split the cluster along the first principal component of the gradients. 

After cluster split and associated embedding optimizations, the current cluster assignment might have room for improvement, e.g., the embedding of some other cluster might better fit the interactions of an entity. Thus, fine-grained automatic cluster optimization also requires enabling cluster reassignment. Given that the training loss appears to be the only referable metric, we propose to conduct cluster reassignment based on direct loss minimization. 

With such a cluster split and reassignment paradigm, CEL is ideally a sustainable framework that can not only support starting with zero data, but also can automatically increase the number of clusters, as well as automatically optimize the cluster reassignments, as the interaction data increases and the number of users/items increases.

\begin{figure*}[ht]
    \centering
    \includegraphics[width=0.96\linewidth]{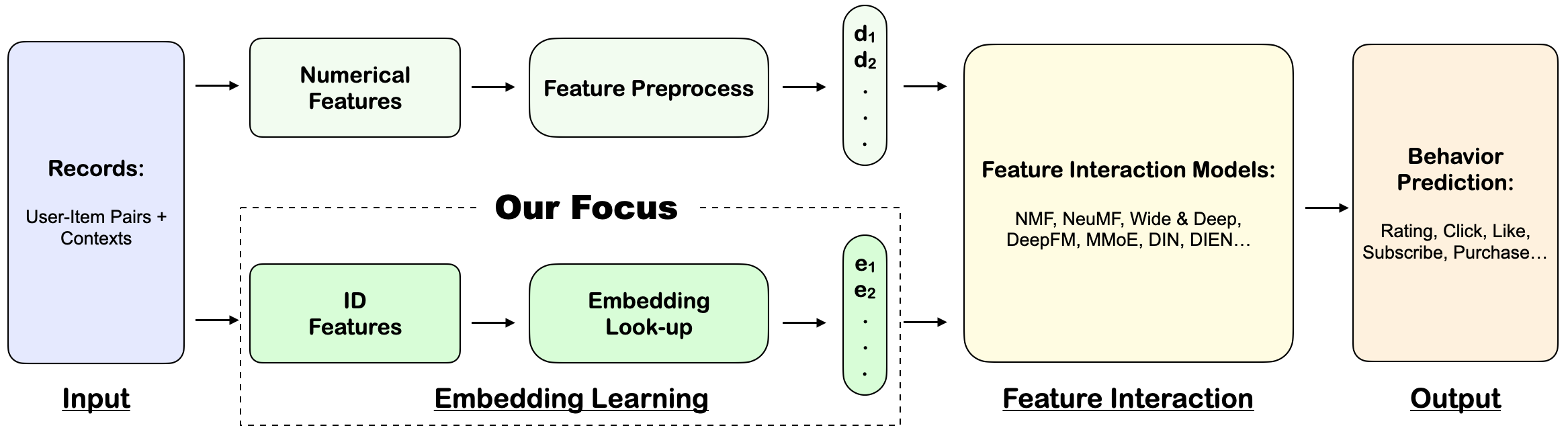}
    \caption{The standard behavior prediction paradigm in recommender system.}
    \label{fig:rankframework}
\end{figure*}

We validate the proposed CEL both theoretically and empirically. Firstly, we prove the identifiability and the existence of a unique optimal number of clusters for CEL in the context of nonnegative matrix factorization. Then, we show that the accelerated version of CEL (CEL-Lite) has an optimal time complexity, which allows efficient online updates. Finally, we conduct extensive experiments on public real-world datasets as well as a private business dataset, integrated with various feature interaction models (including NMF, DIN, and DeepFM); the results show that CEL consistently outperforms current state-of-the-art methods. In particular, when incorporating
CEL into the business model, it brings an improvement of $+0.6\%$ in AUC, which translates into a significant revenue gain; meanwhile, the size of the embedding table gets 2650 times smaller.

\section{Related Works}

One line of research regularizes the embeddings to follow a given clustering structure to improve their qualities. 
For example, Joint NMF and K-Means (JNKM) \cite{jnkm} forces the embeddings to follow a K-clustering structure; eTree \cite{eTREE} learns the embeddings under an implicit tree (i.e., hierarchical clustering). Prior works in this line also include {HSR} \cite{HSR}, {HieVH} \cite{sun2017exploiting}, etc. 
The major distinction of our work is that CEL does not predefine the clustering structure. Instead, it automatically optimizes the number of clusters and the cluster assignment. Besides, these methods require training in full embedding size, while CEL is trained with reduced embedding size.

Another line of work uses the hash to share and jointly learn the embeddings. The plain version applies the modulo operation \cite{modulo} on the hash ID of each entity to obtain its hash code (i.e., the embedding index). Sharing embedding within each hash bucket, the embedding size can be reduced by the magnitude of the modulo divisor. 
Binary Hash (BH) \cite{binaryhash,zhang2016discrete} adopts a code block strategy on binarized hash IDs to reduce direct collisions. 
Adaptive Embedding (AE) \cite{DeepRec} explicitly avoids collisions for warm users who have sufficient historical data. 
However, the hash assignments in these methods are essentially random and will not adapt/learn during learning. 
There are also several works that tried to learn a hash function. For example, HashGNN \cite{hashgnn} learns a binary code hash function with graph neural networks. However, it requires pretraining a hash function. 
Preference Hash (PreHash) \cite{prehash} learns to hash users with similar preferences to the same bucket. PreHash binds a warm user to each bucket, while the embedding of other users is a weighted sum of the embeddings of their top-$K$ related buckets. By contrast, CEL does not force bindings between users and clusters. 

Some other works seek to reduce memory cost by reducing the embedding dimension or the number of bits for each dimension (i.e., quantization) \cite{Mixeddimension,Automatedembedding,amtl,zhang2016discrete,yang2019quantization}, which are orthogonal to our work in the perspective of reducing the memory cost. 
The type of clustering algorithms most relevant to CEL is Divisive Hierarchical Clustering (DHC) \cite{dhcreview}. 
Particularly, Principal Direction Divisive Partitioning (PDDP) \cite{pddp} splits clusters via principal coordinate analysis \cite{PCoA}, which is remotely related to our proposed GPCA (Section~\ref{sec:gpca}).

\begin{figure*}[ht]
    \centering
    \includegraphics[width=0.98\linewidth]{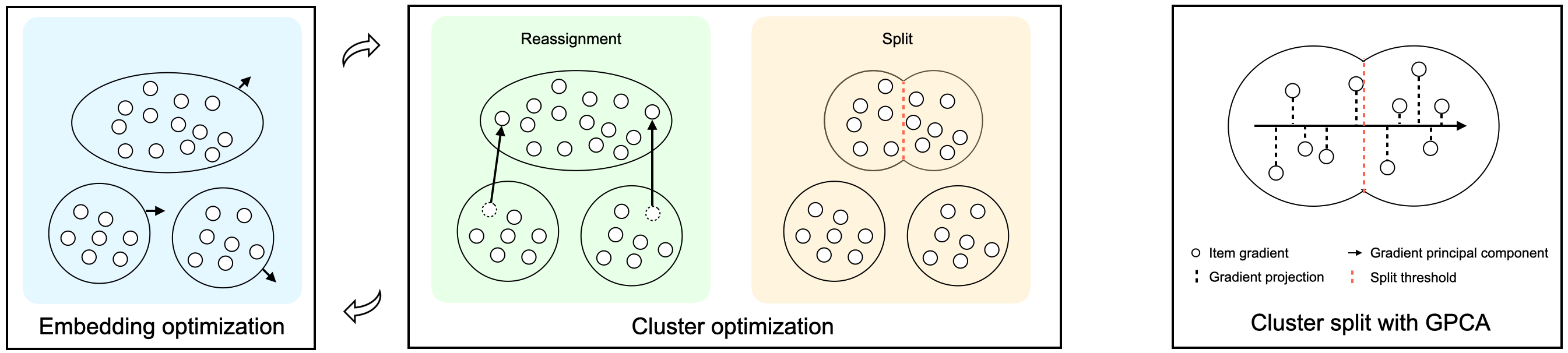}
    \\
    \hspace{35mm}(a)\hspace{80mm}(b)
    \caption{(a) An overview of CEL: it alternatively performs embedding optimization and cluster optimization. (b) The illustration of GPCA: it splits clusters along the first principal component of item gradients.}
    \label{fig:celframework}
\end{figure*}
\section{The Proposed Framework} \label{sec:methodology}

In this section, we present Clustered Embedding Learning (CEL) as a plug-and-play embedding learning framework which can be integrated with any differentiable feature interaction model such as Nonnegative Matrix Factorization (NMF), NeuMF~\cite{neuralCF}, Wide\&Deep \cite{cheng2016wide}, DeepFM~\cite{guo2017deepfm}, MMoE~\cite{ma2018modeling}, DIN~\cite{din} and DIEN~\cite{dien}. We illustrate the role of embedding learning in the behavior prediction  that is commonly used in recommender systems in Figure~\ref{fig:rankframework}.

Assume we have a user-item interaction data matrix $\mathbf{X}\in\mathbb{R}^{N\times M}$ of $N$ users and $M$ items, where $\mathbf{X}(i,j)$ denotes the interaction (rating, click, etc.) of the $i^{th}$ user on the $j^{th}$ item. We aim to recover the data matrix with two low-rank matrices, i.e.,  $\mathbf{X}\approx Y(\mathbf{A},\mathbf{B})$, where $\mathbf{A}\in\mathbb{R}^{N\times R}$ and $\mathbf{B}\in\mathbb{R}^{M\times R}$ can be regarded as the user embedding matrix and the item embedding matrix, respectively. Each row of $\mathbf{A}$ (or $\mathbf{B}$) corresponds to the embedding of a user (or item). $R$ is the embedding dimension. 
The prediction $Y(\mathbf{A},\mathbf{B})=[y_{ij}]=[y(\mathbf{A}_i,\mathbf{B}_j)]\in\mathbb{R}^{N\times M}$ is produced by a differentiable feature interaction model $y$ that takes the embedding of user $i$ and the embedding of item $j$ as inputs. 
Generally, embedding learning can be formulated as an optimization problem: 
\begin{equation}
    \argmin_{\mathbf{A},\mathbf{B}} \; \mathcal{L},
    \;
    \mathcal{L}=\big\|\mathbf{W}\odot(\mathbf{X}-Y(\mathbf{A},\mathbf{B}))\big\|_{F}^2 ,
\label{equ:nmf_objective}
\end{equation}
where $\mathbf{W} \in \{0, 1\}^{N\times M}$ contains 1s at the indices of observed entries in $\mathbf{X}$, and 0s otherwise. $F$ denotes the Frobenius norm and $\odot$ denotes element-wise multiplication. 

CEL clustering can be similarly applied to users and items. For clarity, we will only discuss the clustering of items. To initiate the clustering of items, we decompose the item embedding matrix as a product following a clustering structure: 
\begin{equation}
    \mathbf{B}=\mathbf{S}_q\mathbf{B}_q,\;\;q=\{0, 1, 2, \ldots\},
\label{equ:sqbq}
\end{equation}
where $\mathbf{S}_q\in \{0, 1\}^{M\times M_q}$ and $\mathbf{B}_q\in\mathbb{R}^{M_q\times R}$ denote the cluster assignment matrix and the cluster embedding matrix of items, respectively. $M_q$ denotes the current number of clusters and $q$ denotes the current number of splits. We adopt hard cluster assignment: each row of $\mathbf{S}_q$ has exactly one $1$, while the other entries are zero. 

CEL solves Eq.~\eqref{equ:nmf_objective} by alternatively performing embedding optimization and cluster optimization (Figure~\ref{fig:celframework}).

\subsection{Embedding Optimization}

In embedding optimization, we optimize $\mathbf{A}$ and $\mathbf{B}_q$ subject to a fixed $\mathbf{S}_q$.
The optimization of $\mathbf{A}$ can be directly conducted with Eq. \eqref{equ:nmf_objective} and \eqref{equ:sqbq}. Since hard cluster assignment is adopted, the embedding of each cluster can be optimized separately. 

The loss of the $k^{th}$ cluster embedding of $\mathbf{B}_q$ can be denoted as: 
\begin{equation}
\begin{split}
    \mathcal{L}_k=\big\|\mathbf{W}\odot\big[&\mathbf{X}\cdot \text{diag}(\mathbf{S}_q(:,k))
    \\
    &-Y(\mathbf{A},\mathbf{S}_q(:,k)\mathbf{B}_q(k,:))\big]\big\|_F^2,
\end{split}
\end{equation}
where colon $(:)$ stands for the slicing operator that takes all values of the respective field, and $\text{diag}(\cdot)$ maps a vector into a diagonal matrix. It has 
$\mathcal{L}=\sum_{k=1}^{M_q}\mathcal{L}_k$.

\subsection{Cluster Optimization}

In this section, we introduce the two proposed cluster optimization operations: cluster reassignment and cluster split. We assume that the reassignment occurs every $t_1$ steps of embedding optimization, while the split occurs every $t_2$ steps of embedding optimization. 

\subsubsection{Cluster Reassignment}

An item assigned to a sub-optimal cluster can result in performance loss. When the embeddings have been optimized, we may perform cluster reassignment so that an item can be assigned to a better cluster. Reassignment is an important step toward automatically grouping similar items.

One problem CEL faces is how to measure the distance between an item and a cluster, since their representations are in different domains: items have interaction data, while clusters have cluster embeddings. A natural metric seems to be the fitness of the cluster embedding to the item. Thus, we seek to directly minimize the loss of each item w.r.t. their cluster assignments:
\begin{equation}
    \mathbf{S}_q =
    \argmin_{\mathbf{S}_q} \; \big\|\mathbf{W}\odot(\mathbf{X}-Y(\mathbf{A},\mathbf{S}_q\mathbf{B}_q))\big\|_{F}^2.
\label{equ:reassignment}
\end{equation}
With fixed embeddings, the reassignment of each item is independent, therefore Eq. \eqref{equ:reassignment} can also be equivalently formulated as a set of sub-problems, one for each item:
\[
\mathbf{S}_q(j,:) =
\argmin_{\mathbf{S}_q(j,:)} \; \big\|\mathbf{W}(:,j)\odot(\mathbf{X}(:,j)-Y(\mathbf{A},\mathbf{S}_q(j,:)\mathbf{B}_q))\big\|_{F}^2.
\]

\subsubsection{Cluster Split with GPCA} \label{sec:gpca}

With embeddings optimized, we may consider cluster split to achieve a finer clustering. Each time we select one cluster and split it into two. There are many heuristic criteria for choosing the cluster. For example, we may choose the cluster that has the largest: total loss, mean loss, number of associated items, or number of associated interactions ($\|\mathbf{W}\mathbf{S}_q(:,k)\|_0$). According to our experiments, they generally lead to similarly good results. This demonstrates the robustness of CEL upon cluster-choosing criteria. In this paper, we adopt the number of associated interactions as the default criterion.

The key challenge for CEL is how to split a cluster. With hard cluster reassignment and embedding sharing, embeddings within a cluster are the same. Thus we need a representation that can differentiate items in the same cluster. 
To tackle this problem, we propose Gradient Principal Component Analysis (GPCA) that splits clusters according to items' gradients w.r.t. their shared embedding, which indicates their different desired optimization directions of the embedding. 
Such a split strategy can help achieve finer-grained clustering of similar items while further minimizing the loss. 

Specifically, we assemble the gradients of items of a given cluster into a matrix $\mathbf{G}=[\ldots, g_j,\ldots]^\top$, $\forall j\in\{j|\mathbf{S}_q(j,k)=1\}$, where $g_j=-\partial\mathcal{L}/\partial \mathbf{B}(j,:)$ and $k$ denotes the cluster to be split. 
Then, we normalize $\mathbf{G}$ (as a standard procedure before performing PCA, without changing its notation) and compute its first principal component: 
\begin{equation}
\mathbf{p}=\argmax_{\|\mathbf{p}\|_2=1}\;\mathbf{p}^\top\mathbf{G}^\top\mathbf{G}\mathbf{p}.
\label{equ:pcavector}
\end{equation}
After that, we split the cluster into two according to their first principal component scores $g_j^\top\mathbf{p}$:
\begin{equation}
\hspace{-2mm}
    \left\{
        \begin{array}{ll}
          \mathbf{S}_{q+1}(j,k)=1,
          \;\;\;\;\;\;\;\;\;\; \text{ if }g_j^\top\mathbf{p}< \delta; \\
          \mathbf{S}_{q+1}(j,M_q+1)=1,
          \;\; \text{ if }g_j^\top\mathbf{p}\geq \delta.
        \end{array}
    \right.
\label{equ:pca score}
\end{equation}
One set of items is still assigned to cluster $k$, while the other set is assigned to a new cluster $M_q+1$. We may adjust the threshold $\delta$ to control/balance the number of items in each cluster. 

We present the pseudocode of CEL in Algorithm \ref{alg:cel}.

\subsection{Theoretical Analysis of CEL under NMF}
\label{sec:identifiability}

CEL can be integrated with any differentiable feature interaction model, however, most of them (e.g., when involving other features or integrated with neural networks) are hard to analyze. Even so, we are able to study the theoretical properties of CEL when it is integrated with Nonnegative Matrix Factorization (NMF). These analyses can to some extent demonstrate the validity of CEL. 

In NMF, $Y(\mathbf{A}, \mathbf{B}) = \mathbf{A} \mathbf{B}^\top$ and $\mathbf{A}, \mathbf{B}\geq 0$, which factorizes the data matrix $\mathbf{X}$ into the product of two low-rank matrices. Despite that low-rank matrix factorization generally has no unique solution, nonnegativity helps establish the identifiability of NMF. NMF can thus produce essentially unique low-dimensional representations \cite{lee1999learning}. The NMF of $\mathbf{X} = \mathbf{A} \mathbf{B}^\top$ is regarded as \emph{essentially unique} if the nonnegative matrix $\mathbf{A}$ and $\mathbf{B}$ are identifiable up to a permutation and scaling. We show that the solution of CEL is identifiable: 

\begin{theorem}
(Identifiability of CEL) Assume that the data matrix $\mathbf{X}$ follows $\mathbf{X} = \mathbf{A} \mathbf{B}_q^\top \mathbf{S}_q^\top$, where $\mathbf{A} \in\mathbb{R}^{N\times R}$, $\mathbf{B}_q\in \mathbb{R}^{M_q\times R}$, $\mathbf{S}_q \in \{0,1\}^{M\times M_q}$, $\|\mathbf{S}_q(j,:)\|_0 = 1,\forall j\in [1,M]$, $rank(\mathbf{X}) = rank(\mathbf{A}) = R$, $\mathbf{B}_q$ has no repeated rows and, without loss of generality, $M_q\geq R$. If $\mathbf{S}_q$ is of full-column rank and rows of $\mathbf{B}_q$ are sufficiently scattered, then $\mathbf{A}$, $\mathbf{B}_q$ and $\mathbf{S}_q$ are essentially unique.
\label{theorem:CELIdentifiability}
\end{theorem}

The formal definitions (e.g., sufficiently scattered) \cite{eTREE,jnkm,identifiability} and detailed proofs are provided in Appendix 
B. 

We assume $\mathbf{S}_q$ to be of full-column rank, which means there is no empty cluster. Enforcing such a condition in our implementation has improved the model performance. 
The assumption of $\mathbf{B}_q$ being sufficiently scattered is common and is likely to be satisfied as shown by \cite{eTREE}. 
Based on Theorem~\ref{theorem:CELIdentifiability}, we can further deduce the existence of a unique optimal number of clusters:

\begin{proposition}
(Optimal Number of Clusters of CEL) Under the assumptions of Theorem \ref{theorem:CELIdentifiability} and further assume that the data matrix can also be decomposed as $\mathbf{X} = \mathbf{A}' \mathbf{B}_{q+1}^\top \mathbf{S}_{q+1}^\top$, where $\mathbf{A}' \in\mathbb{R}^{N\times R}$, $\mathbf{B}_{q+1}\in \mathbb{R}^{M_{q+1}\times R}$, $\mathbf{S}_{q+1} \in \{0,1\}^{M\times M_{q+1}}$, $\|\mathbf{S}_{q+1}(j,:)\|_0 = 1,\forall j\in [1,M]$, $rank(\mathbf{A}') = R$, and $M_{q+1}>M_{q}\geq R$. If $\mathbf{S}_{q+1}$ is of full-column rank and the rows of $\mathbf{B}_{q+1}$ are sufficiently scattered, then $\mathbf{B}_{q+1}$ has only $M_{q}$ distinct rows, which can be obtained via permutation and scaling of the rows in $\mathbf{B}_q$.
\label{proposition:clusternumber}
\end{proposition}

The proof is built on the fact that the identifiability of CEL implies $\mathbf{S}_{q+1}\mathbf{B}_{q+1}$ is a permutation and scaling of $\mathbf{S}_{q}\mathbf{B}_{q}$. The implication of Proposition \ref{proposition:clusternumber} is twofold: 1) there exists a unique optimal number of clusters\footnote{A formal proof is provided in Appendix~B.}; 2) once the optimal number of clusters is reached, further split is unnecessary and will only introduce duplicated cluster embeddings.

Based on this observation and given that we do not have the exact formulation of the optimal number, we will keep splitting, while always initializing the embedding of the newly split cluster to be the same as the chosen cluster (this will retain the optimality of the previous solution). Besides, in online learning settings where we continually receive streaming data, the optimal number of clusters may keep increasing. In this sense, cluster split is well-motivated and makes CEL a sustainable learning framework.

\subsection{CEL-Lite for Efficient Online Update}
\label{sec:online}

\subsubsection{Online Learning.}

Modern recommender systems generally deal with a massive amount of data. 
Such systems are thus usually updated in an online fashion: data arrive in sequence over time, and updates take place when the accumulated data reaches a predefined quantity. This framework closely follows the online learning paradigm and allows the model to rapidly adapt to the newest data. 

We adapt CEL into CEL-Lite for the online learning setting by optimizing its efficiency in dealing with batch data. CEL-Lite has a theoretically optimal time complexity and inherits most of the advantages of CEL. CEL-Lite is also suitable when we have access to the full data but want to train the model in an online fashion. 

\subsubsection{Data and Embedding Optimization.}

Under the online setting, each new batch $\mathbf{X}_b$ consists of a set of user-item interactions $\{\mathbf{X}(i,j)\}$. To enable the usage of interaction history as well as mitigate \emph{catastrophic forgetting} \cite{forgetting}, we buffer from online batches a small number of the latest interactions ($\leq\!n$) from each user and item for \emph{experience replay} \cite{replay}, which is commonly used in online systems. Since each interaction is represented by a triplet of $\langle$user Id, item Id, rating$\rangle$, the storage cost incurred by buffering is acceptable. Moreover, the interactions can be stored in distributed file systems, unlike the embeddings or model parameters which must be stored in memory for efficient updates.

In CEL-Lite, one step of embedding optimization is performed for each batch via stochastic gradient descent with related buffered interactions. The overall time complexity of embedding optimization is then at most $\mathcal{O}(nDR)$, where $D$ denotes the total number of interactions $\|\mathbf{W}\|_0$.

\subsubsection{Efficient Cluster Reassignment.}
 
To achieve fast cluster reassignment, we incorporate the following accelerations into CEL-Lite: 1) only considers $m$ additional randomly sampled clusters for each reassignment; 2) only items involved in the current batch are considered for reassignment.

With these two strategies, the cluster assignments are only improved locally. Nevertheless, with sufficient updates, it may be able to converge to the global optimal. Accordingly to our experiments, such accelerations significantly improve efficiency, while only incurring a small drop in performance. The time complexity of each reassignment is now only $\mathcal{O}(nmbR)$, where $b$ is the batch size. Given there are $D/b$ batches, there are at most $D/b$ cluster reassignments. Thus, the overall time complexity is $\mathcal{O}(nmDR)$. 

\subsubsection{Efficient Cluster Split.}

To control the overall complexity of cluster split for CEL-Lite, we propose to constrain the number of associated interactions in each cluster and ensure clusters are approximately balanced. These two conditions limit the maximal number of split operations and the time complexity of each split. Concretely, we set up the following split strategies: 
\begin{strategy} 
If an item has more than $d$ interactions, it will directly split and form a cluster by itself.
\label{strategy:1}
\end{strategy}
\begin{strategy}
A cluster will split if and only if it has more than $2d$ associated interactions ($\|\mathbf{W}\mathbf{S}_q(:,k)\|_0$). The split should be balanced such that the difference between the number of associated interactions of resulting clusters is within $d$. 
\label{strategy:2}
\end{strategy}

Strategy~\ref{strategy:1} handles extremely warm items. After executing strategy~\ref{strategy:1}, each clustered item has no more than $d$ interactions; strategy~\ref{strategy:2} then gets to easily balance the splits, which can be done by tuning\footnote{The tuning of $\delta$ can be done by \emph{Quickselect} in $\mathcal{O}(\|\mathbf{S}_q(:,k)\|_0)$.} the threshold $\delta$ in Eq.~\eqref{equ:pca score}.

With strategy \ref{strategy:1} and \ref{strategy:2}, all clusters have at least $d/2$ interactions. Therefore, there are at most $2D/d$ clusters and hence at most $2D/d$ GPCA splits in total. Calculating the gradient of each item with only their $n$ buffered interactions, the time complexity of a single GPCA split is approximately $\mathcal{O}(n\|\mathbf{S}_q(:,k)\|_0R)$, which is no more than $\mathcal{O}(2ndR)$. The overall time complexity of splits is thus $\mathcal{O}(4nDR)$.

Note that we use a constant $d$ in strategies \ref{strategy:1} and \ref{strategy:2} for simplicity and ease of use in practice. It can in fact generalize to any monotonically nondecreasing sublinear functions of $D$ without altering the overall time complexity. 
These strategies allow the maximum number of clusters to grow according to $D$, which practices our theoretical results in Section \ref{sec:identifiability}.

\subsubsection{The Optimal Time Complexity}
\label{sec:time complexity}

Regarding $n,m$ as constants, each of the three optimizations of CEL-Lite is of a time complexity of $\mathcal{O}(DR)$. Thus, the overall complexity of CEL-Lite is also $\mathcal{O}(DR)$. 
Note that once we need to process all the data and each invokes an update of the embedding, the time complexity is already $\mathcal{O}(DR)$. That is, the time complexity of $\mathcal{O}(DR)$ in a sense is the lower bound of embedding learning.

\section{Experiments}

\begin{table}[t]
\centering
\caption{Datasets summary.}
\resizebox{.42\textwidth}{!}{
\begin{tabular}{lrrr}
\hline
Dataset        & Users & Items & No. interactions \\ \hline

MovieLens-1m   & $6,040$  & $3,952$  & $1,000,209$ \\

Video Games  & $826,767$ & $50,210$ & $1,324,753$ \\
Electronics  & $4,201,696$ & $476,002$ & $7,824,482$ \\
One-day Sales  & $\approx 13$m & $\approx40$m & $\approx700$m \\\hline
\end{tabular}}
\label{tab:datasets}
\end{table}

In this section, we empirically analyze CEL and compare it with related embedding learning baselines, including: Vanilla full embedding; eTree \cite{eTREE}; JNKM \cite{jnkm}; Modulo \cite{modulo} (i.e., with plain hash); BH \cite{binaryhash}; AE \cite{DeepRec}; HashGNN \cite{hashgnn}; and PreHash \cite{prehash}. 
They are tested with a variety of feature interaction models, including: NMF, DeepFM~\cite{guo2017deepfm}, DIN~\cite{din}, and a business model from an E-commerce platform. 

The datasets considered include: MovieLens, a review dataset with user ratings ($1\!\sim\!5$) on movies; Video Games and Electronics, review datasets with user ratings ($1\!\sim\!5$) of products from Amazon; and One-day Sales, sub-sampled real service interaction data within a day of  the E-commerce platform Shopee. The key statistics of these datasets are given in Table~\ref{tab:datasets}.


The original eTree and JNKM do not share embedding within clusters. To make the comparison thorough, we adapt eTree such that in the test phase, leaf nodes (users/items) will share the embedding of their parent nodes. Similarly, we adapt JNKM such that clustered items share the embedding of their cluster in the test phase. PreHash has a hyperparameter $K$ that selects top-$K$ related clusters for each user or item. We report its results with different values of $K$. 

We reuse $M_q$ to also denote: the number of parent nodes in eTree; the number of clusters in JNKM; and the number of hash buckets in hash-based methods. We refer to $M_q/M$ as the \textbf{compression ratio}.

\begin{table*}[htb]
\centering
\caption{The test Mean Square Error (MSE) of rating predictions on MovieLens-1m dataset (averaged over 5-fold cross validation). A compression ratio (defined as $M_q/M$) of $100\%$ means the result is obtained under full embedding, while other compression ratios indicate that the results are obtained with clustered/hashed/shared embedding. Results marked with (p) indicate the results obtained with personalization (Section~\ref{sec:personalization}).}
\resizebox{0.81\textwidth}{!}{
\begin{tabular}{c|ccc|cccc|cccc}
\hline
\multicolumn{4}{c|}{
Feature Interaction Model
} & \multicolumn{4}{c|}{NMF} & \multicolumn{4}{c}{DIN} \\ \hline
\multicolumn{4}{c|}{Compression Ratio}       
             &$100\%$ &$5\%$ &$1\%$  &$0.5\%$ &$100\%$ &$5\%$ &$1\%$   &$0.5\%$       \\ \hline
 & LTC & LMC & SIT & & & & & & & &
\\ \hline
Vanilla full embedding & \cmark & \xmark & \cmark & $0.8357$ & - & -  & -  
                  & $0.7903$ & - & -  & - \\ 
eTree & \xmark & \xmark & \xmark        & $0.7841$ & $0.8399$ & $0.8482$ & $0.8614$ 
             & $0.7640$ & $0.8306$ & $0.8379$ & $0.8470$ \\
JNKM& \xmark & \xmark & \xmark        & $0.7626$ & $0.8903$ & $0.8616$ & $0.8498$ 
             & $0.7486$ & $0.8868$ & $0.8560$ & $0.8466$ \\ 
Modulo & \cmark & \cmark & \cmark    & - & $1.0227$ & $1.0688$  & $1.0732$  
             & - & $1.0206$ & $1.0647$ & $1.0692$ \\
BH     & \cmark & \cmark & \cmark     & - & $0.9753$ & $1.0170$  & $1.0526$  
             & - & $0.9907$ & $1.0040$ & $1.0495$ \\
AE    & \cmark & \cmark & \cmark       & - & $0.9959$ & $1.0343$  &  $1.0630$ 
             & - & $0.9895$ & $1.0280$ & $1.0624$ \\  
HashGNN   & \xmark & \xmark & \xmark   & - & $0.8491$ & $0.8829$ & $0.8882$ 
             & - & $0.8467$ & $0.8663$ & $0.8723$ \\ 
PreHash ($K=4$)   & \cmark & \cmark & \xmark & - & $0.8169$ & $0.8832$ & $0.8987$ 
                    & - & $0.8057$ & $0.8678$ & $0.8897$ \\ 
PreHash ($K=M_q$) & \xmark & \cmark & \xmark & - & $0.7958$ & $0.8047$ & $0.8101$ 
                    & - & $0.7871$ & $0.7893$ & $0.7942$ \\ 
CEL (Ours)  & \xmark & \cmark & \cmark     & $\textbf{\small{\textsf{0.7507}}}^p$ & $\textbf{\small{\textsf{0.7784}}}$ & $\textbf{\small{\textsf{0.7858}}}$ & $\textbf{\small{\textsf{0.7926}}}$ & $\textbf{\small{\textsf{0.7390}}}^p$ & $\textbf{\small{\textsf{0.7718}}}$ & $\textbf{\small{\textsf{0.7787}}}$ & $\textbf{\small{\textsf{0.7868}}}$ \\
CEL-Lite (Ours) & \cmark & \cmark & \cmark  & $\mathbf{0.7519}^p$ & $\mathbf{0.7820}$ & $\mathbf{0.7901}$ & $\mathbf{0.8014}$ & $\mathbf{0.7421}^p$ & $\mathbf{0.7766}$ & $\mathbf{0.7789}$ & $\mathbf{0.7906}$ \\ \hline
\end{tabular}}
\label{tab:movielens ranting prediction}
\end{table*}

\subsection{Performance Versus Baselines}
\subsubsection{Rating Prediction}

In this experiment, we perform rating prediction on the MovieLens datasets. We adopt NMF as well as DIN (both with $R=64$) as our feature interaction models, respectively. We set $t_1=40$ and $t_2=10$ for CEL, while $b=2000$, $n=20$, $m=10$, $t_1=1$ for CEL-Lite. 
Since some baselines apply their methods only on items or users, to ensure a fair comparison, for this experiment we only compress (clustering, hashing, etc.) the item embeddings while keeping full user embeddings for all methods.

The results are shown in Table~\ref{tab:movielens ranting prediction}.
It can be observed that CEL outperforms existing baselines under all settings. According to our results, training with predefine clustering structure, eTree and JNKM brings clear improvements over vanilla methods in the full embedding setting. However, they consistently underperform CEL, especially when using a compressed embedding table. For hash based methods, the improvements brought by BH and AE (over naive Modulo) are relatively limited, while HashGNN and PreHash achieve more clear performance gains. Among them, PreHash with $K=M_q$ achieves the best performance in the compressed settings. Nevertheless, CEL clearly outperforms PreHash. 
Meanwhile, the results of this experiment demonstrate that CEL-Lite, being much more computationally efficient than CEL, has a similar performance to CEL and still clearly outperforms other baselines. The degradation in performance might be attributed to its subsampling accelerations. 

In addition, we check in Table~\ref{tab:movielens ranting prediction} three important properties for each method: Low Time Consumption (\textbf{LTC}), which indicates whether it achieves a time complexity of $\mathcal{O}(DR)$ for both training and test. In particular, algorithms whose time complexity scales linearly with $M_q$ is not regarded as being LTC. Thus, eTree, JNKM, PreHash (with $K=M_q$), and CEL is not regarded as LTC. Neither does HashGNN, as it requires running GNN aggregation of neighborhoods, which has a worst case time complexity exponential in $D$; Low Memory Consumption (\textbf{LMC}), which indicates whether the method can significantly reduce the memory cost for both training and testing compared with full embedding. JNKM, eTree, and HashGNN require a full embedding table during training, thus do not satisfy LMC; Support Incremental Training (\textbf{SIT}), which indicates whether the method supports incremental training, is crucial for practical systems. As the data increases, the predefined tree/clustering structure of eTree/JNKM can be sub-optimal, the pretrained hash function of HashGNN may become outdated, and the pre-specified pivot users in PreHash may become less representative. As a result, these methods may require periodical retraining.

\begin{table*}[ht]
\centering
\caption{The test AUC (presented in \% for a better interpretation) on Video Games and Electronics datasets.}

\resizebox{0.99\textwidth}{!}{
\begin{tabular}{c|cccc|cccc|cccc|cccc}
\hline
Dataset      & \multicolumn{8}{c|}{Video Games}  & \multicolumn{8}{c}{Electronics} \\ \hline
Feature Interaction      & \multicolumn{4}{c|}{DeepFM}  & \multicolumn{4}{c|}{DIN} & \multicolumn{4}{c|}{DeepFM}  & \multicolumn{4}{c}{DIN} \\ \hline
Compression Ratio     & $100\%$    & $10\%$    & $5\%$    & $2\%$   & $100\%$    & $10\%$    & $5\%$    & $2\%$   & $100\%$    & $10\%$    & $5\%$    & $2\%$   & $100\%$    & $10\%$    & $5\%$& $2\%$   \\ \hline
Vanilla full embedding      & $68.19$ & - & - & - & $70.73$ & - & - & - 
            & $66.95$ & - & - & - & $67.62$ & - & - & -  \\
Modulo      & - & $61.68$ & $58.81$ & $56.34$ 
            & - & $63.35$ & $61.94$ & $60.25$ 
            & - & $63.13$ & $54.54$ & $48.61$ 
            & - & $63.70$ & $61.69$ & $58.93$ \\
HashGNN     & - & $68.51$ & $64.43$ & $61.84$ 
            & - & $71.59$ & $71.11$ & $70.82$ 
            & - & $64.49$ & $61.41$ & $59.45$ 
            & - & $65.50$ & $63.49$ & $60.40$ \\
PreHash ($K=M_q$)       & - & $69.19$ & $66.21$ & $63.89$ 
                        & - & $72.72$ & $71.63$ & $71.46$ 
                        & - & $67.50$ & $67.18$ & $66.99$ 
                        & - & $70.30$ & $68.80$ & $68.77$ \\
CEL-Lite    & $\mathbf{70.85}^p$ & $\mathbf{70.33}$ & $\mathbf{70.21}$ & $\mathbf{69.82}$ 
            & $\mathbf{73.85}^p$ & $\mathbf{73.80}$ & $\mathbf{73.50}$ & $\mathbf{73.39}$ 
            & $\mathbf{68.98}^p$ & $\mathbf{68.90}$ & $\mathbf{68.69}$ & $\mathbf{68.03}$ 
            & $\mathbf{71.63}^p$ & $\mathbf{71.56}$ & $\mathbf{70.18}$ & $\mathbf{69.81}$ \\ \hline
\end{tabular}}
\label{tab:cr predictions}
\end{table*}
\subsubsection{Online Conversion Prediction}

\label{exp:Conversion Rate}

To promote desirable behaviors from users, such as clicks and purchases, conversion rate prediction is an important learning task in online E-commerce systems. In this section, we test the models with the conversion prediction task, i.e., predicting whether a user will perform a desirable behavior on an item. We use the Video Games and Electronics datasets for this experiment, treating ratings not less than $4$ as positive feedback (i.e., conversions). For performance evaluation, we use the standard Area under the ROC Curve (AUC) on the test sets, which can be interpreted as the probability that a random positive example gets a score higher than a random negative example. 

The experiment is conducted in an online fashion, such that the data is received incrementally over time. We adopt HashGNN and PreHash, the top-2 performing hash methods in the rating prediction experiment (as well as the naive Modulo) as our baselines, though they do not support incremental training. The hash function of HashGNN requires pretraining with full data before starting the actual embedding learning, thus we allow it to access the full data before the online training. PreHash requires binding the warmest users to its buckets, we thus allow it to go through all data to find the warmest users before the online training. It is worth noting that, unlike these two methods, CEL(-Lite) can naturally start embedding learning from scratch without any extra information. 

In this experiment, we compress both the user embeddings and the item embeddings. To further demonstrate that our method works well with different feature interaction models, we adopt DeepFM in this experiment, while keeping using DIN. We set $R=8$ and $b=256$. We set $n=20$, $m=50$, and $t_1=1$ for CEL-Lite. 

The results are shown in Table~\ref{tab:cr predictions}. It can be observed that CEL-Lite still clearly outperforms the baselines.

\begin{table}[ht]
\centering
\caption{The test AUC  (\%) on One-day Sales dataset.}
\resizebox{.35\textwidth}{!}{
\begin{tabular}{c|cccc}
\hline
Dataset                & \multicolumn{4}{c}{One-day Sales} \\ \hline
\begin{tabular}[c]{@{}c@{}}Embedding size \\  
\end{tabular} & Full          & 1m       & 100k        & 10k         \\

\hline
Business Model                 & $87.09$      & $86.80$     & $86.55$          & $86.46$          \\
CEL-Lite                    & -           & -          & $\mathbf{87.78}$     & $\mathbf{87.70}$     \\ \hline
\end{tabular}}
\label{tab:one day sales}
\end{table}

\subsubsection{Integrated with Business Model}

To further validate the performance of CEL-Lite, we conduct experiments on dataset from the E-commerce platform Shopee, that needs to handle billions of requests per day. We use the model deployed in that platform as our baseline, which is a fine-tuned business model that has a sophisticated network structure and heavy feature engineering. 
It optionally uses full embedding, or uses Modulo to compress the size of embedding table. In this experiments, we integrate CEL-Lite into this business model, replacing its embedding learning module. 

As shown in Table~\ref{tab:one day sales}, CEL-Lite with a small number of clusters (with 10k or 100k embedding entities for users and items, respectively) outperforms the business model with full embedding, and more significantly, the business model with compression. 
Specifically, CEL-Lite achieves an improvement of $+0.61\%$ in AUC with the embedding table size being 2650 times smaller. 
Note that a $+0.1\%$ gain in absolute AUC is regarded as significant for the conversion rate tasks in practice \cite{cheng2016wide,din}. 

\begin{figure}[h]
    \centering
    \includegraphics[width=.45\textwidth]{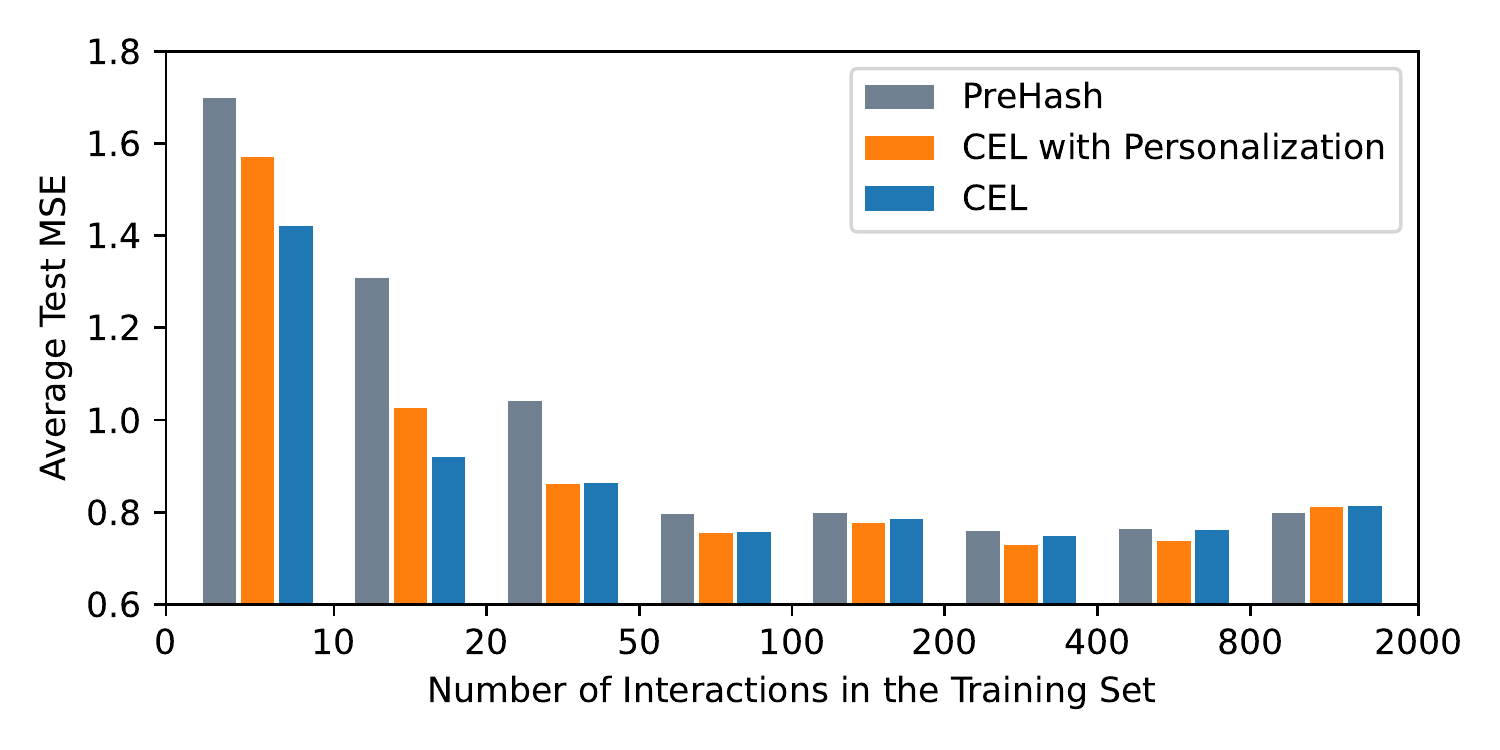}
    \caption{Averaged test MSE of items with different number of interactions. PreHash (with $K=M_q$) is the best-performing baseline in Table~\ref{tab:movielens ranting prediction}. CEL significantly outperforms PreHash on cold items. Personalization degenerates the performance of cold items, while improves the performance of warm items.}
    \label{fig:loss vs frequency}
\end{figure}

\begin{figure*}[ht]
    \centering
    \hspace{-1mm}
    \includegraphics[width=57mm]{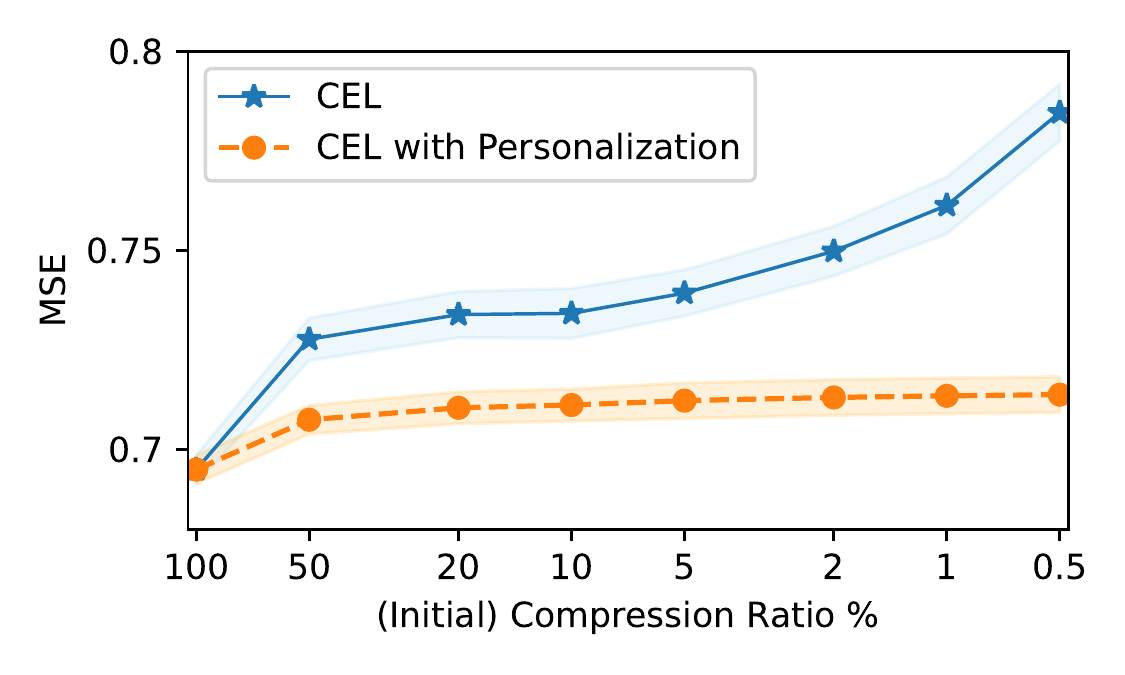}
    \includegraphics[width=56mm]{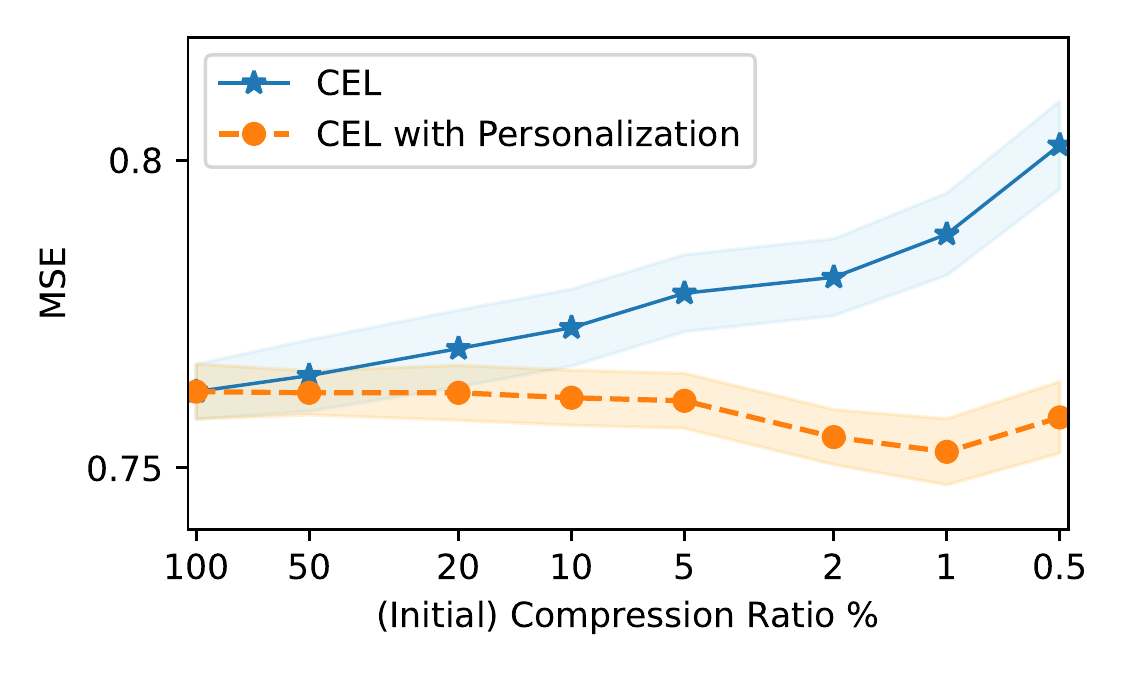}
    \includegraphics[width=57mm]{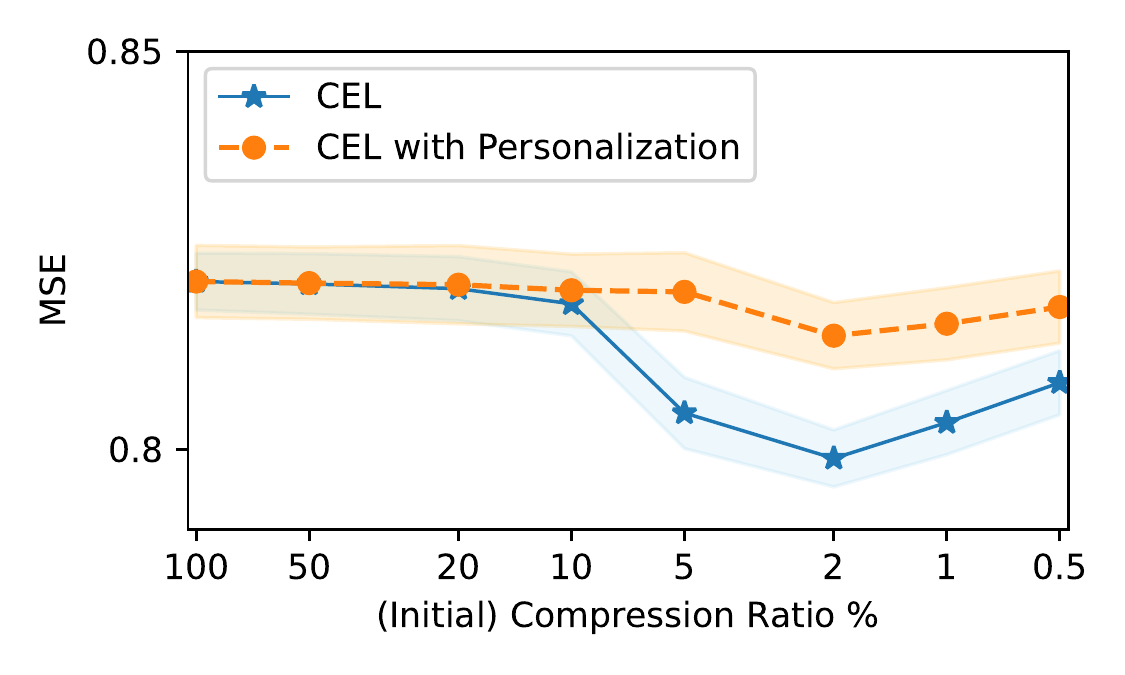}
    \hspace{6mm}(a)\hspace{52mm} (b)\hspace{53mm} (c)
    \caption{The test MSE of different (initial) compression ratios on (a) the \emph{dense} MovieLens-1m, (b) the \emph{original} MovieLens-1m, and (c) the \emph{sparse} MovieLens-1m. Standard errors over $5$ runs are indicated by the shades. CEL with personalization at a given compression ratio indicates the fully personalized result, initiated with CEL at that compression ratio. $100\%$ means gradually split to full embedding.}
    \label{fig:sparsity vs ratio}
\end{figure*}

\subsection{Personalization and Cold-start Problems}
\subsubsection{Personalization.} \label{sec:personalization}

When training with shared embeddings, we may choose to untie the cluster assignments at a specific point to learn a personalized embedding (with the shared embedding as initialization) for each user and item. After being untied, embeddings can be further optimized according to their own historical data. This makes sense if we want a full embedding, in case we have enough data and memory. We refer to this procedure as \emph{personalization}, which is similar to the personalization in federated learning settings \cite{tan2022towards}. There are various personalization techniques. We leave them for future investigations. Here, we simply add the following distance regularization to the original objective to guide the personalization: 
\[\mathcal{L}_{p}=\frac{\lambda_p}{2}\sum_{j=1}^M\|\mathbf{B}(j,:)-\mathbf{S}_q(j,:)\mathbf{B}_q\|_2^2.\]
We set $\lambda_p=50$ in our experiment.
Table~\ref{tab:movielens ranting prediction}\&\ref{tab:cr predictions} shows CEL with personalization achieves the best results. In particular, it has outperformed all other full embedding methods.

\subsubsection{Mitigating Cold-start Problems.}

We investigate the test loss of items with respect to the number of associated interactions in the training set (on the MovieLens-1m dataset, with a compression ratio of $1\%$). As shown in Figure~\ref{fig:loss vs frequency}, CEL significantly outperforms PreHash on cold items. This could be because PreHash emphasizes the warm items, while the cold items are roughly represented as a weighted sum of the warm ones. This may also be why PreHash performs slightly better in the interval of $[800,2000]$. On the other hand, we can see that if we do personalization for CEL, the performance of cold items degenerates. This strongly evidences those cold items cannot be well optimized with only their own data, and CEL can help in this regard. Personalization brings benefits for the intervals of $>\!20$, hence can have a better overall performance. We further try only personalizing the warm items ($>\!20$), and find a $1.2\%$ improvement on test MSE. 

\subsubsection{Sparsity and Compression Ratio.}

To investigate the impact of data sparsity and compression ratio on CEL, we create a dense MovieLens-1m dataset by filtering out the users and items with no more than $100$ interactions, and a sparse MovieLens-1m dataset by randomly dropping $50\%$ interactions (but still ensuring at least $1$ sample) for each user. 

As shown in Figure~\ref{fig:sparsity vs ratio}, on the dense and original datasets, increasing the number of clusters consistently improves the overall performance of CEL; while on the sparse dataset, the optimal compression ratio occurs at $2\%$: further splits will degenerate its performance. 

We can also see that personalization at different compression ratios consistently brings improvement over CEL on the dense and original datasets, which is reasonable as warm items may have dominated the overall performance in these datasets. By contrast, on the sparse dataset, personalization consistently leads to degeneration of performance. 
More specifically, on the dense dataset, increasing the number of clusters consistently improves the performance of personalization, while on the original and sparse dataset, there is an optimal compression ratio of $1\%\sim2\%$ for personalization. 

Overall, the number of clusters is not the more the better. Too many clusters on a sparse dataset or a dataset that has many cold items, may lead to performance degeneration: partially on cold items or even the overall performance.

\subsection{Ablation Studies on Clustering}
\subsubsection{Interpretation of Clustering.}

To see how and how well CEL clusters items, we visualize the learned clusters of CEL on MovieLens-1m at $M_q=10$ in Figure~\ref{fig:cel_interpretation}. 

We find that at this stage, the clusters are mainly divided according to the ratings of items. The standard deviations of ratings within clusters are generally low. This result is reasonable because the task is to predict the ratings. 

We can also see that movies within a cluster appear to have some inner correlation. For example, cluster $6$ mainly consists of action or science fiction action movies; cluster $5$ mainly consists of music or animation movies that are suitable for children; cluster $9$ mainly consists of crime movies that have gangster elements; and cluster $8$ contains $70\%$ movies directed by Stanley Kubrick. Note that such information is not fed to the model in any form, but learned by CEL automatically. 
A calculation of the averaged entropy of the distribution of genres in each cluster further evidences that CEL can better distinguish movies of different genres than others. 

Interestingly, we find movies of the same type may form multiple clusters of different average ratings. For example, clusters $4$ and $6$ are both mostly action, but with a clear difference in their ratings.

\begin{table}[htb]
\centering
\caption{The test MSE of rating predictions on MovieLens-1m. We compare CEL with rule-based clusterings, and test using rule-based clusterings as initialization for CEL.}
\scalebox{0.92}{
\begin{tabular}{cc|ccc}
\hline
\multicolumn{2}{c|}{Number of clusters ($M_q$)} & $18$ & $40$ & $200$      \\ \hline
Initialization & Continue Running & && \\ \hline
By Genre &  - & $1.0515$ & - & - \\
By Avr. Rating &  -  & $0.8397$ & $0.8390$ & $0.8380$ \\
By Genre       & CEL & $0.8066$ & $0.7859$ & $\mathbf{0.7776}$ \\
By Avr. Rating & CEL & $\mathbf{0.7959}$ & $0.7892$ & $0.7833$ \\
None           & CEL & ${0.7985}$ & $\mathbf{0.7858}$ & $0.7784$ \\ \hline 
\end{tabular}}
\label{tab:by_rule}
\end{table}

\subsubsection{Rule-based Clusterings}

CEL achieves automatic clustering of entities, but it is still not clear how it works against rule-based clusterings. Given that we find in the initial stage CEL mainly clusters items according to the genres and ratings of movies on the MovieLens dataset. In this section, we compare CEL with rule-based clustering based on genres (partitioning movies according to their genres in the meta data; there are 18 genres in total) and ratings (partitioning movies evenly according to their averaged rating). 

As shown in Table~\ref{tab:by_rule}, neither of these two rule-based clusterings works well.
We further try using these two rule-based clusterings as initialization for CEL, to see whether it can help improve the performance. As we can see, such initialization brings either no or insignificant improvements.

\begin{figure*}[ht]
    \centering
    \includegraphics[width=1\linewidth]{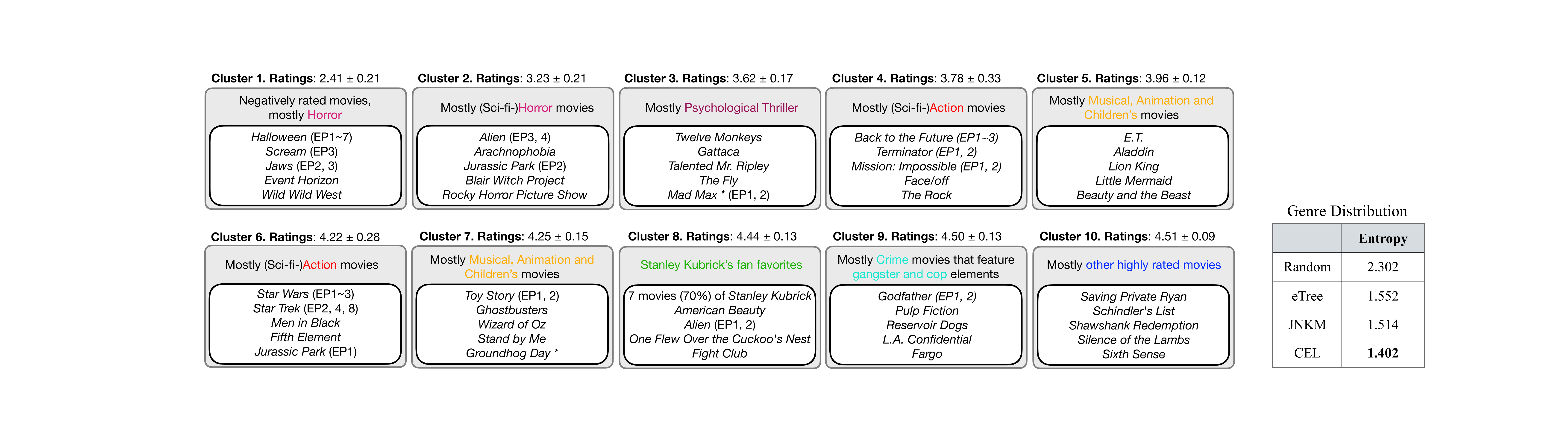}
    \caption{Clusters learned by CEL, ordered by their averaged ratings. We provide a one-line description for each cluster by summarizing the movies contained inside, and list 5 movies with the most number of ratings, with outliers marked by *. The table on the right shows the averaged entropy of the genre distribution in the clusters, calculated with the meta-data available.}
    \label{fig:cel_interpretation}
\end{figure*}

\subsubsection{Split Methods} 

\begin{table}[htb]
\centering
\caption{Test MSE of different split methods.}
\scalebox{0.92}{
\begin{tabular}{ccc}
\hline
MovieLens         & 100k & 1m \\
\hline
Random Split      &  $0.9392$ & $0.8516$   \\
User-Inspired Bisecting 2-Means &  $0.9245$ & $0.8355$   \\
Gradient Bisecting 2-Means &  $0.8847$ & $0.8186$   \\
Gradient Random Projection &  $0.9201$  & $0.8269$   \\
GPCA              &  $\mathbf{0.8707}$ & $\mathbf{0.7858}$\\
\hline
\end{tabular}}
\label{tab:split methods}
\end{table}

We test several different split methods, including: Random Split, which randomly splits items into two sets; Bisecting 2-Means, an advanced split method from DHC \cite{steinbach2000comparison}, which bisects items into 2 clusters according to the distance between item representations; Random Projection, which is similar to GPCA but projects the gradients in a random direction instead of the first principal component; and GPCA. 
Bisecting 2-Means is performed either on the gradient matrix or on the user-inspired representations (i.e., representing an item by averaging the embeddings of its associated users). For all methods, we split till a compression ratio of $1\%$. 

The results are shown in Table~\ref{tab:split methods}. We can see that random split yields the worst result, which indicates the importance of split methods; Bisecting 2-Means works better on gradients than user-inspired representations, which suggests the gradients are more informative; Bisecting 2-Means and random project of gradient works less than GPCA, which suggests the rationality of split according to the first principal component. 

\section{Conclusions}

In this paper, we proposed Clustered Embedding Learning (CEL), a novel embedding learning framework that can be plugged into any differentiable feature interaction model. It is capable of achieving improved performance with reduced memory costs. We have shown the theoretical identifiability of CEL and the optimal time complexity of CEL-Lite. Our extensive experiments suggested that CEL can be a powerful component for practical recommender systems. 

We have mainly considered the embeddings of users and items in this paper. Nevertheless, CEL can be similarly applied to other ID features or categorical features, such as shop ID or seller ID in E-commerce, search words in search engines, object categories, and geographic locations. We leave these for future work. 

\section*{Acknowledgments}
This research is supported in part by the National Natural Science Foundation of China (U22B2020); the National Research Foundation Singapore and DSO National Laboratories under the AI Singapore Programme (AISG2-RP-2020-019); the RIE 2020 Advanced Manufacturing and Engineering Programmatic Fund (A20G8b0102), Singapore; the Nanyang Assistant Professorship; the Shanghai University of Finance and Economics Research Fund (2021110066). 

\bibliographystyle{ACM-Reference-Format}
\bibliography{sample-base}

\appendix

\begin{table*}[htb]
\centering
\caption{The test MSE of rating predictions on MovieLens datasets with NMF as the feature interaction model.} 
\resizebox{0.90\textwidth}{!}{
\begin{tabular}{c|cccc|cccc|cccc}
\hline
Datasets & \multicolumn{4}{c|}{MovieLens-100k} &\multicolumn{4}{c|}{MovieLens-1m} & \multicolumn{4}{c}{MovieLens-10m} \\ \hline
Compression Ratio        
             &$100\%$ &$5\%$ &$1\%$   &$0.5\%$ &$100\%$ &$5\%$ &$1\%$   &$0.5\%$ &$100\%$ &$5\%$ &$1\%$   &$0.5\%$       \\ \hline

Vanilla full embedding       & $0.9283$ & - & -  & - & $0.8357$ &  - & -& -  & $0.7227$ & - & - & - \\
eTree        & $0.9127$ & $0.9830$ & $0.9973$ & $0.9873$  & $0.7841$ & $0.8399$ & $0.8482$ & $0.8614$ & $0.7000$ & $0.7701$ & $0.7830$ & $0.7902$ \\
JNKM         & $0.8841$ & $1.0501$ & $1.0463$  & $1.1054$ & $0.7626$ & $0.8903$ & $0.8616$ & $0.8498$  & $0.6941$ & $0.7766$ & $0.7923$ & $0.7962$ \\ 
Modulo       & - & $1.0908$ & $1.1045$  & $1.1097$  & - & $1.0227$ & $1.0688$  & $1.0732$  & - & $0.8402$ & $0.8532$ & $0.8658$ \\
BH           & -      & $1.0243$ & $1.0695$  & $1.0729$ & - & $0.9753$ & $1.0170$  & $1.0526$   & -  & $0.8351$ & $0.8441$ & $0.8544$ \\
AE           & -      & $1.0715$ & $1.0863$  & $1.0969$ & - & $0.9959$ & $1.0343$  &  $1.0630$ & -  & $0.8377$ & $0.8503$ & $0.8637$ \\
PreHash ($K=4$)     & - & $0.9165$ & $0.9602$ & $0.9833$ & - & $0.8169$ & $0.8832$ & $0.8987$ & - & $0.7816$ & $0.7934$ & $0.7937$ \\ 
PreHash ($K=M_q$)   & - & $0.9140$ & $0.9533$ & $0.9599$ & - & $0.7958$ & $0.8047$ & $0.8101$ & - & $0.7493$ & $0.7518$ & $0.7692$ \\ 
CEL (Ours)        & $\mathbf{0.8602}^p$ & $\mathbf{0.8689}$ & $\mathbf{0.8707}$ & $\mathbf{0.8906}$ & $\mathbf{0.7507}^p$ & $\mathbf{0.7784}$ & $\mathbf{0.7858}$ & $\mathbf{0.7926}$ & $\mathbf{0.6888}^p$ & $\mathbf{0.7301}$ & $\mathbf{0.7374}$ & $\mathbf{0.7389}$  \\
CEL-Lite (Ours)     & $0.8611^p$ & $0.8702$ & $0.8824$ & $0.8928$& ${0.7519}^p$ & $0.7820$ & $0.7901$ & $0.8014$  & ${0.6904}^p$ & $0.7343$ & $0.7407$ & $0.7421$ \\ \hline
\end{tabular}}
\label{tab:movielens ranting prediction old}
\end{table*}

\section{Appendix}

In this appendix, we provide some important implementation details and some additional ablation studies. Detailed proofs about identifiability (Appendix~B), detailed time complexity analysis and proofs (Appendix~C), and more implementation details and results (Appendix~D), can be found in \url{https://arxiv.org/abs/2302.01478}.

\begin{algorithm}
\small
  \caption{The pseudocode of CEL}\label{alg:cel}
  \SetKwInOut{Input}{Input}\SetKwInOut{Output}{Output}
  \Input{User-item interaction data matrix $X$}
  \Output{User embedding matrix $\mathbf{A}$, item cluster assignment matrix $\mathbf{S}_q$, item cluster embedding matrix $\mathbf{B}_q$, and trained model $y$}
  \BlankLine
  {Initialize $M_q\gets 1$, $\mathbf{S}_q \gets  [1]^{M\times M_q}$, set hyperparameters $\{\delta,\eta\}$}\;
    {Randomly initialize $\mathbf{A}$, $\mathbf{B}_q$, and trainable parameters in $y$ }\;
  
  \While{not converged}{
  $\mathcal{L} \gets \big\|\mathbf{W}\odot(\mathbf{X}-Y(\mathbf{A},\mathbf{S}_q\mathbf{B}_q))\big\|_{F}^2$\;
  $\mathbf{A}\gets \mathbf{A}- \eta \cdot \partial\mathcal{L}/\partial \mathbf{A}$, \quad  $\mathbf{B}_q\gets \mathbf{B}_q- \eta \cdot \partial\mathcal{L}/\partial \mathbf{B}_q$\;
    
      \If{satisfy reassign condition}{
          $\mathbf{S}_q \gets \argmin_{\mathbf{S}_q} \; \big\|\mathbf{W}\odot(\mathbf{X}-Y(\mathbf{A},\mathbf{S}_q\mathbf{B}_q))\big\|_{F}^2$\;
      }
      \If{satisfy split condition}{
          $k \gets \argmax_{k}\|\mathbf{W}\mathbf{S}_q(:,k)\|_0$  \;
          $\mathbf{B}_q(M_q+1,:)\gets  \mathbf{B}_q(k,:)$ \tcp*[r]{New row of $\mathbf{B}_q$}
          $\mathbf{S}_q(:,M_q+1) \gets [0]^{M}$ \tcp*[r]{New column of $\mathbf{S}_q$}
          Update $\mathbf{S}_q$ with GPCA\tcp*[r]{Eq.~\eqref{equ:pcavector} and ~\eqref{equ:pca score}}
          
            $M_q \gets M_q+1$\;
      }
      
    }
\end{algorithm}


%

\subsection{More Details about Experiment Settings}
\label{appendix:hypers}

We apply stochastic gradient descent with a learning rate $0.0001$ and $0.05$ for CEL and CEL-Lite, respectively. We also adopt the gradient averaging technique (Appendix D.3). The total number of iterations and other hyperparameters on each dataset is tuned by validation. 

By default, we set the initial number of clusters $M_0$ to be $1$ in all of our experiments. We have tuned $M_0$ (when $M_0>1$, as initialization, the users/items are mapped randomly into $M_0$ clusters) and find a generally consistent performance when $M_0$ is relatively small ($\leq10$). 

When splitting with GPCA, we use a priority queue to keep track of the clusters.
We set the cluster splitting threshold $d=100$, and the split is stopped when the total number of clusters $M_q$ reaches the given compression ratio.

We perform $5$-fold cross validation for the rating prediction task, and train with $5$ random seeds for other tasks, reporting the average.

For the public datasets used in this paper, we use a split ratio of $80\%/20\%$ for dividing the train set and the test set. For Amazon datasets, we report the test performance of items/users which have interaction(s) on the training set.

In our implementation, DIN and DeepFM utilize the sparse features 
from the available meta data. Such sparse features allow feature crossing in DeepFM. 
Following the common heuristic~\cite{neuralCF}, we use $16$ and $8$ neurons in the two hidden layers of DIN, respectively. We use the same network architecture for DeepFM. The business model has a sophisticated network structure and heavy feature engineering, which leads to more than $10$m trainable non-embedding parameters.

On MovieLens datasets, the genre of each movie is provided in the metadata. We define the the genre distributions, i.e., a distribution $p_g(x_k),k\in[1,M_q]$ for genre $g$ in $M_q$ clusters.
We thus can calculate the averaged entropy (e.g., in Figure~\ref{fig:cel_interpretation}) of the genre distribution in the clusters, which can be expressed below
\[\textstyle
\text{Entropy}=\frac{1}{|\text{Genres}|}\sum_{g\in\text{Genres}}\sum_{k=1}^{M_q}p_g(x_k)\log p_g(x_k).\]
Lower entropy indicates a better distinction of genres.

\begin{figure*}[t]
    \centering
    \includegraphics[width=.25\textwidth]{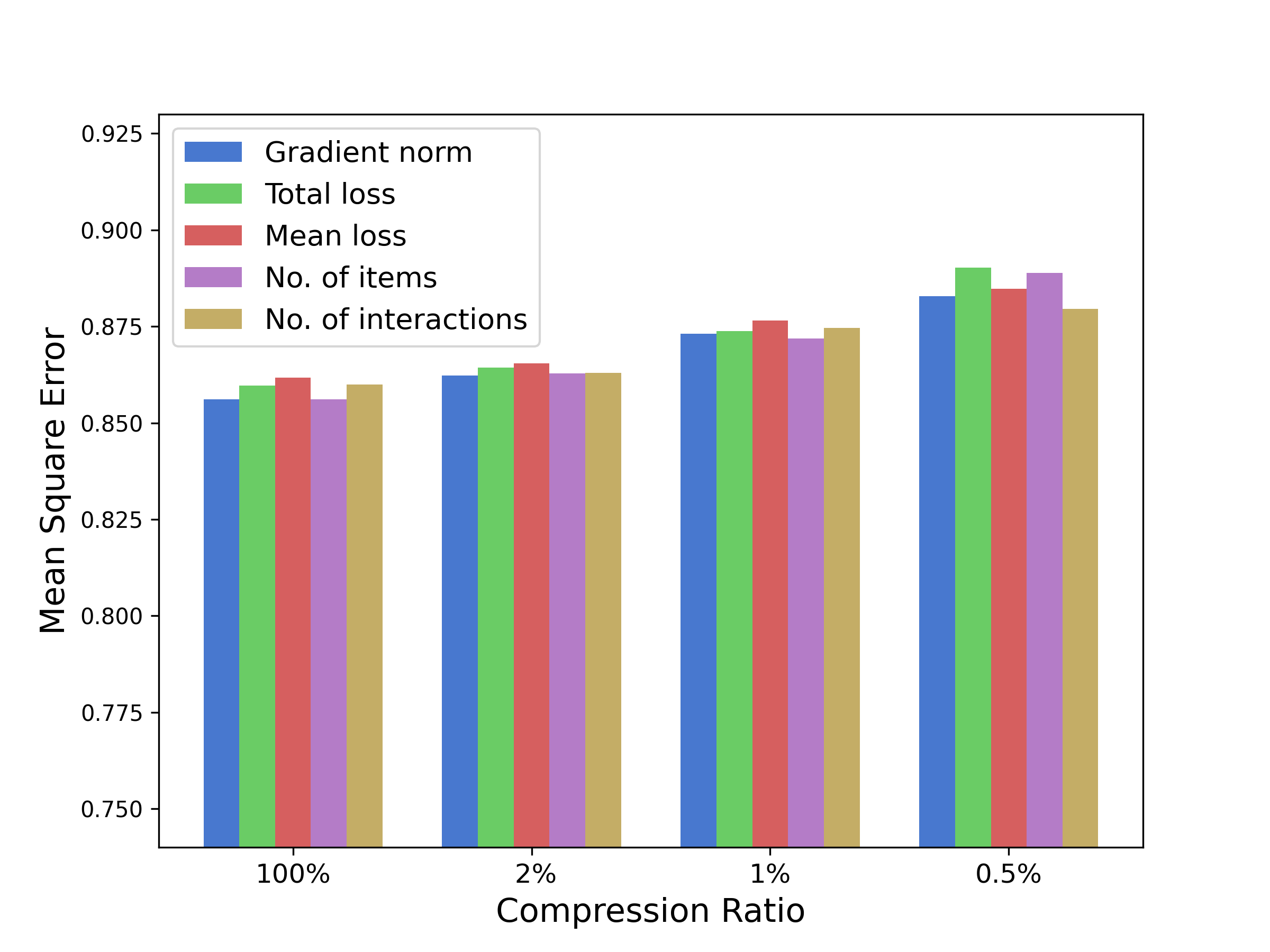}
    \includegraphics[width=.25\textwidth]{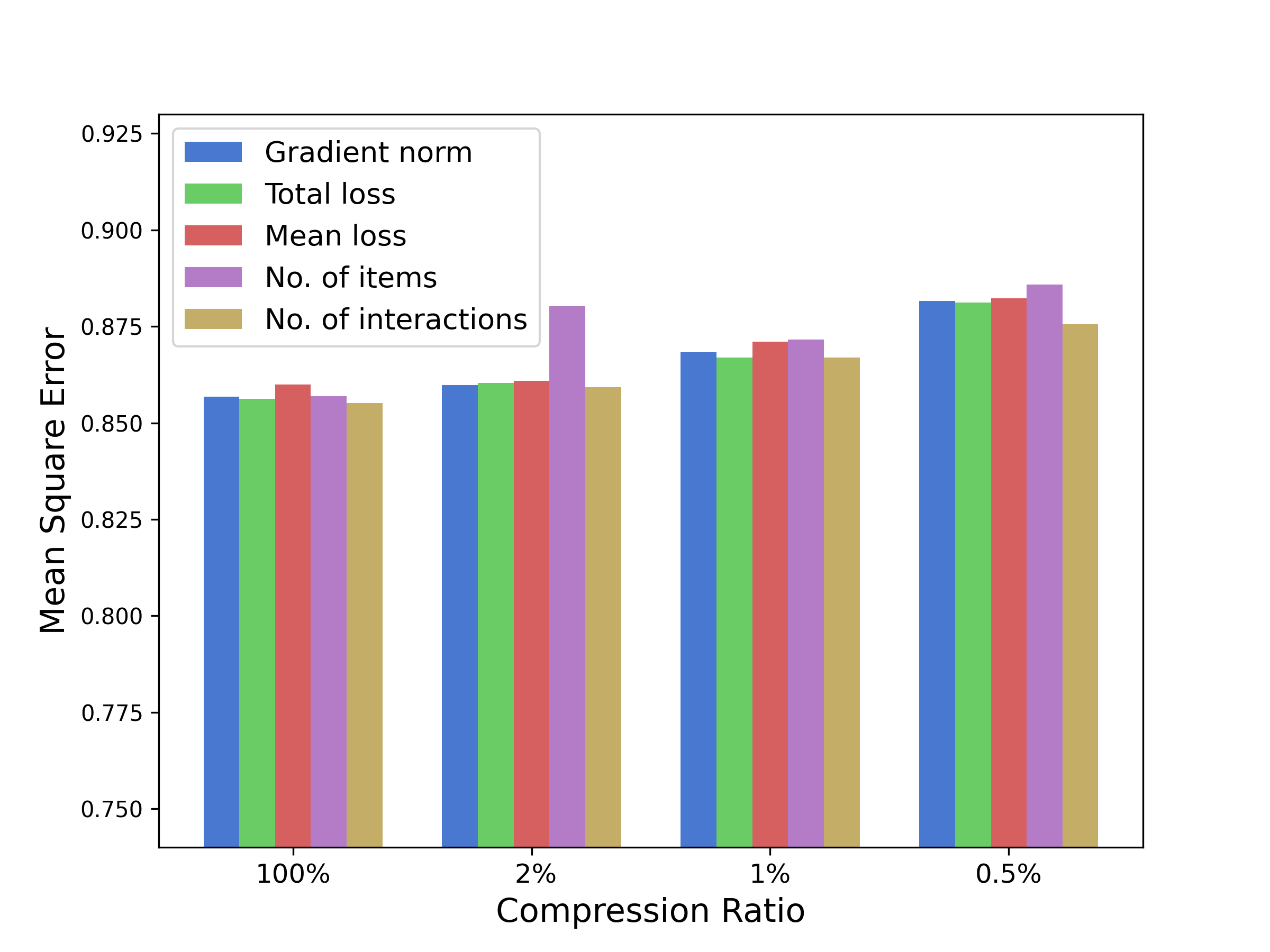}
    \includegraphics[width=.25\textwidth]{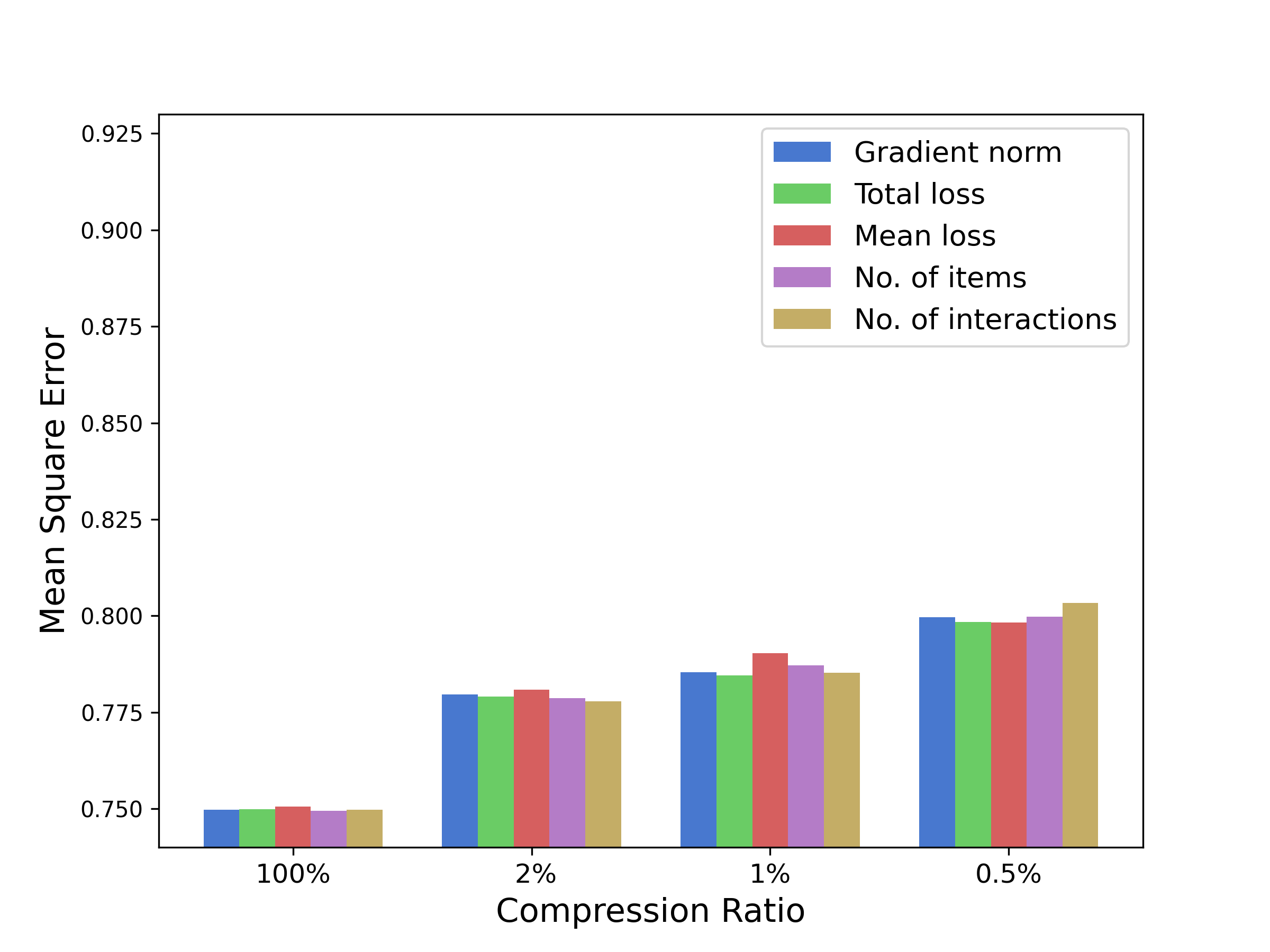}
    \raisebox{2mm}{
    \includegraphics[width=.22\textwidth]{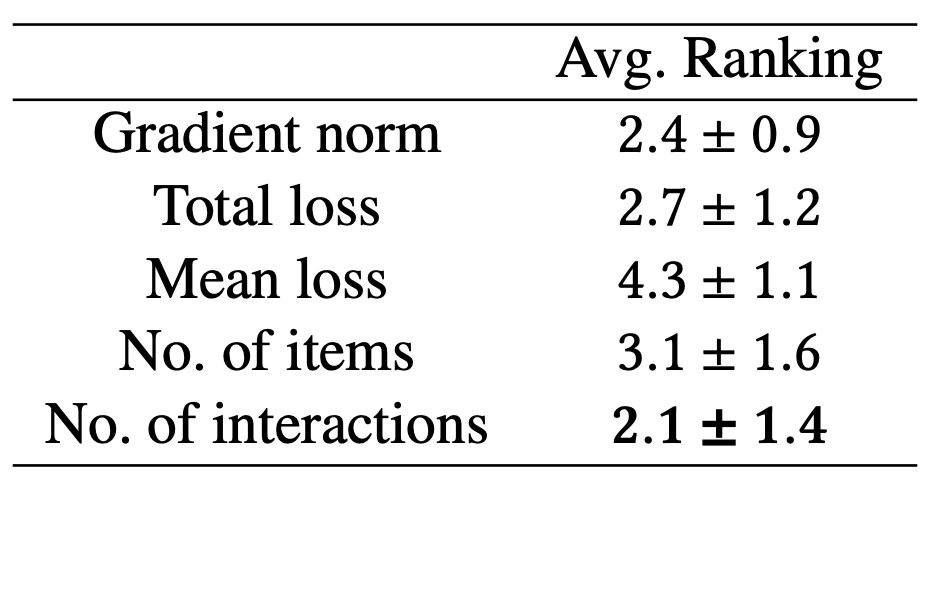}}\\
   \hspace{3mm} (a) \hspace{40mm} (b) \hspace{40mm} (c) \hspace{36mm} (d)
    \caption{The test MSE of CEL under different settings (averaged over $5$ runs): (a) MovieLens-100k + NMF; (b) MovieLens-100k + DIN; (c) MovieLens-1m + NMF. There are totally $5$ cluster choosing criteria: 1) The total gradient norm of a cluster (computed as $\text{trace}(\mathbf{G}_k^\top\mathbf{G}_k)$);\ 2) The total loss of a cluster $\mathcal{L}_k$;\ 3) The averaged loss of a cluster $\mathcal{L}_k/\|\mathbf{S}_q(:,k)\|_0$;\ 4) The number of the associated item $\|\mathbf{S}_q(:,k)\|_0$;\ 5) The number of the associated interactions (data) $\|\mathbf{W}\mathbf{S}_q(:,k)\|_0$. On (d) we also compute the averaged ranking (i.e., from $1^{st}$ the best to $5^{th}$ the worst) of each cluster choosing criterion on all the $12$ settings tested.}
    \label{fig:bar}
\end{figure*}   

\subsection{On Datasets of Different Scales}
Besides MovieLens-1m, we also performed experiments on the dataset of MovieLens-100k, MovieLens-10m, and MovieLens-25m, which contains 100k, 10m, and 25m interactions respectively. As the results shown in Table~\ref{tab:movielens ranting prediction old} and \ref{tab:movielens 25m}, CEL consistently outperforms baselines. Other observations are also in line with our discussion in the main paper. 

\begin{table}[htb]
\centering
\tiny
\caption{The test MSE of rating predictions on MovieLens-25m datasets with NMF as the feature interaction model.}
\resizebox{0.46\textwidth}{!}{
\begin{tabular}{c|cccc}
\hline
Datasets & \multicolumn{4}{c}{MovieLens-25m} \\ \hline
Compression Ratio        
             &$100\%$ &$5\%$ &$1\%$   &$0.5\%$       \\ \hline

Vanilla full embedding       & $0.8359$ & - & - & - \\
Modulo       & - & $1.0782$ & $1.1680$ & $1.2004$ \\
PreHash (\scalebox{0.8}{$K=M_q$})   & - & $ 0.9095 $ & $ 0.9566 $ & $ 1.0104 $ \\
CEL-Lite (Ours)   & $\mathbf{0.7508}$\scalebox{0.8}{$^p$} & $\mathbf{0.7809}$ & $\mathbf{0.8119}$ & $\mathbf{0.8404}$ \\ \hline
\end{tabular}}
\label{tab:movielens 25m}
\end{table}

\subsection{Ablation Studies: Split}
In this section, we perform a thorough study of the split in CEL.

\label{appendix:split}
\subsubsection{Split vs. No Split}
We first provide empirical evidence of the benefit of CEL split operation by ablations studies. For \textit{no split}, we initialize the number of clusters $M_0$ to be exactly the maximum allowed number.
The assignments are randomly initialized. 
Note that reassignments are still allowed.
The results in Table \ref{tab:splitvsnosplit} show that the split operation does bring benefits compared with no split. Another strong evidence supporting the superior performance of our GPCA split can be found in Table~\ref{tab:split methods}.

\begin{table}[h]
\centering
\caption{The test MSE. Comparison of split vs. no split.}
\scalebox{0.8}{
\begin{tabular}{c|ccc|ccc}
\hline
MovieLens-1m      & \multicolumn{3}{c|}{NMF} & \multicolumn{3}{c}{DIN} \\  \hline
Compression Ratio & $100\%$ &  $1\%$  & $0.5\%$ & $100\%$ &  $1\%$  & $0.5\%$        \\ \hline
CEL (no split) & 0.7582& 0.7991& 0.8025& 0.7428& 0.7824& 0.8023\\
CEL            & 0.7474& 0.7846& 0.7926& 0.7390& 0.7787& 0.7868 \\ \hline 
\end{tabular}}
\label{tab:splitvsnosplit}
\end{table}

\subsubsection{Criterion for Choosing the Split Cluster}
\label{appendix:split criterion}

We test $5$ cluster choosing criteria, results shown in Figure~\ref{fig:bar}. We can see that they generally lead to similarly good performance. It demonstrates the robustness of CEL upon cluster choosing criteria. It is worth noting that the mean loss criterion performs worse than the others, probably because the cold items tend to have larger men loss than the warm items, as revealed by our experiment on Figure~\ref{fig:loss vs frequency}. Thus focusing on the mean may overly emphasize the cluster with cold items, and thus lead to undesired behavior. Thus in this sense, the total loss (or other criteria) may perform better than the mean.

\subsubsection{Threshold for GPCA Split}

We test different thresholds $\delta$ for split with GPCA. The results are shown in Table~\ref{tab:split threshold}, where we split till a compression ratio of $2\%$. As we can see, they generally lead to similarly good performance, indicating the robustness of GPCA split w.r.t the threshold $\delta$. For the other results reported in this paper, we adopt a default value of $\delta=0$. Noted that GPCA always finds a significant $1^{st}$ principal component: the eigenvalue decompositions in our experiments show that the largest eigenvalue is usually $\sim\!100$ times larger than the second largest eigenvalue. 

\begin{table}[hbt!]
\centering
\caption{The test MSE for different split thresholds.}
\resizebox{0.42\textwidth}{!}{
\begin{tabular}{l|cc}
\hline
\multicolumn{1}{c|}{MovieLens}         & 100k & 1m \\
\hline
$\delta=0$                 & \textbf{0.8684} & \textbf{0.7795} \\
$\delta=$ median of the $1^{st}$ principal component scores           & 0.8691 & 0.7807 \\
$\delta$ chosen to balance the split of data (Section~\ref{sec:online}) & 0.8689 & 0.7808\\
\hline
\end{tabular}}
\label{tab:split threshold}
\end{table}

\subsection{Ablation Studies: Reassignment}
\label{appendix:Assignment}
\subsubsection{Reassignment}

We conduct an experiment to quantify the effect of reassignment. The results as in Table~\ref{tab:reassignment} show that under the two scenarios we tested, performing extensive steps of reassignments will lead to slight overfitting, while performing few steps of reassignment will lead to underfitting. Thus the total number of reassignments (as well as $t_1$) is a hyperparameter that can be tuned (e.g., by grid search). This reveals the rationale behind our strategy of a constant number of reassignments. 

\begin{table}[hbt!]
\centering
\caption{The MSE for different frequencies of reassignment. 
}
\resizebox{0.47\textwidth}{!}{
\begin{tabular}{c|ccc|ccc}
\hline
MovieLens-1m & \multicolumn{3}{c|}{Compression Ratio $0.5\%$} & \multicolumn{3}{c}{Compression Ratio $1\%$} \\ \hline
$t_1$  &$20$ &$40$   &$100$ &$20$ &$40$   &$100$  \\ \hline
Total no. of reassignments  &$\approx100$ &$\approx50$   &$\approx20$ &$\approx100$ &$\approx50$   &$\approx20$ \\ \hline
CEL (training MSE)      & 0.7307 & {0.7318} & {0.7329} & {0.7071} & {0.7179} & {0.7187} \\ 
CEL (test MSE)      & 0.7931 & \textbf{0.7926} & 0.7939 & 0.7867 & \textbf{0.7858} & 0.7860 \\ \hline
\end{tabular}}
\label{tab:reassignment}
\end{table}

\subsubsection{Retrain with the Learned Clustering}

We further conduct experiments to retraining with the learned clustering structure. 
The settings of the retraining are identical to the training from which we obtained the clustering structure, so that only difference is that now the assignment is fixed (no cluster split nor cluster reassignment). The results are shown in  Table~\ref{tab:retrain}. As can be seen, the test performance of retraining is nearly identical to that of training, indicating that our CEL algorithm has learned effective clustering structures that can generalize well with different initialization of embeddings. It also implies that CEL does not encounter obvious optimization difficulties during the training to obtain such a clustering structure. 

\begin{table}[hbt!]
\centering
\caption{The test MSE when performing retrain.} 
\scalebox{0.8}{
\begin{tabular}{c|ccc|ccc}
\hline
Datasets & \multicolumn{3}{c|}{MovieLens-100k} & \multicolumn{3}{c}{MovieLens-1m} \\ \hline
Compression Ratio        
             &$100\%$ &$1\%$   &$0.5\%$ &$100\%$ &$1\%$   &$0.5\%$ \\ \hline
CEL       & 0.8671 & {0.8707} & {0.8906} & {0.7507} & {0.7858} & {0.7926} \\ 
CEL (Retrain)      & - & 0.8704 & 0.8905 & - & 0.7858 & 0.7927 \\ \hline
\end{tabular}}
\label{tab:retrain}
\end{table}

\newpage

\section{More about Identifiability of CEL}
\label{appendix:all about Identifiability}
\subsection{Definitions of Related Concepts}
\subsubsection{Essentially Uniqueness.}
We say that the NMF of $\mathbf{X} = \mathbf{A} \mathbf{B}^\top$ is \emph{essentially unique}, if the nonnegative matrix $\mathbf{A}$ and $\mathbf{B}$ are identifiable up to a common permutation and scaling/counter-scaling, that is, $\mathbf{X} = \mathbf{A}' \mathbf{B}'^\top$ implies $\mathbf{A}'=\mathbf{A}\mathbf{\Pi}\mathbf{D}$ and $\mathbf{B}'=\mathbf{B}\mathbf{\Pi}\mathbf{D}^{-1}$ for a permutation matrix $\mathbf{\Pi}$ and a full-rank diagonal scaling matrix $\mathbf{D}$.

\subsubsection{Sufficiently Scatteredness.}
\label{appendix:Definition of Sufficiently Scattered}
To investigate the identifiability, it is a common practice to assume the rows of $\mathbf{B}$ are sufficiently scattered \cite{eTREE,jnkm,identifiability}. Formally, it should have the following property.
\begin{definition} (Sufficiently Scattered) \cite{identifiability} The rows of a nonnegative matrix $\mathbf{B}\in \mathbb{R}^{M\times R}$ are said to be sufficiently scattered if: 1) $cone\{\mathbf{B}^\top\} \supseteq \mathcal{C}$; and 2) $cone\{\mathbf{B}^\top\}^\star \cap bd\{\mathcal{C}^*\}=\{\lambda \mathbf{e}_k|\lambda\geq0, k=1,\ldots,R\}$. 
Here $cone\{\mathbf{B}^\top\}=\{\mathbf{x}|\mathbf{x}=\mathbf{B}^\top\mathbf{\theta},\forall\mathbf{\theta}\geq\mathbf{0},\mathbf{1}^\top\mathbf{\theta}=1\}$ and $cone\{\mathbf{B}^\top\}^\star=\{\mathbf{y}|\mathbf{x}=\mathbf{x}^\top\mathbf{y},\forall\mathbf{x}\in cone\{\mathbf{B}^\top\}\}$ are the conic hull of $\mathbf{B}^\top$ and its dual cone, respectively; $\mathcal{C}=\{\mathbf{x}|\mathbf{x}^\top\mathbf{1}\geq\sqrt{R-1}\|\mathbf{x}\|_2\}$, $\mathcal{C}^*=\{\mathbf{x}|\mathbf{x}^\top\mathbf{1}\geq\|\mathbf{x}\|_2\}$; $bd$ is the boundary of a set.
\end{definition}

The sufficiently scattered condition essentially
means that $cone\{\mathbf{B}^\top\}$ contains $\mathcal{C}$, a central region of the nonnegative orthant, as its subset.

\subsection{Proof of Theorem \ref{theorem:CELIdentifiability}}
\label{appendix:proofoftheorem1}

First we have the following lemma, which is taken from the previous work \cite{identifiability}.

\begin{lemma}
(NMF Identifiability)
Suppose $\mathbf{X} = \mathbf{A} \mathbf{B}^\top$ where $\mathbf{A} \in\mathbb{R}^{N\times R}$ and $\mathbf{B} \in\mathbb{R}^{M\times R}$ are two nonnegative matrices, and $rank(\mathbf{X}) = rank(\mathbf{A}) = R$. If further assume that the rows of $\mathbf{B}$ are sufficiently scattered. Then $\mathbf{A}$ and $\mathbf{B}$ are essentially unique.
\label{lemma:Identifiability}
\end{lemma}
Note that this lemma is a direct implication of Theorem 1 in \cite{identifiability}.

\begin{proof}
\emph{(Proof of Theorem \ref{theorem:CELIdentifiability})}
We assumed the rows of $\mathbf{B}_q$ are sufficiently scattered. Since $\mathbf{S}_q\mathbf{B}_q$ is a permutation with repetition of rows in $\mathbf{B}_q$, rows of $\mathbf{S}_q\mathbf{B}_q$ are sufficiently scattered (because $cone\{\mathbf{B}_q^\top\mathbf{S}_q^\top\}=cone\{\mathbf{B}_q^\top\}$ following from our assumption of $\mathbf{S}_q$ being of full-column rank).
Now, since $A$ is of full-column rank and rows of $\mathbf{S}_q\mathbf{B}_q$ are sufficiently scattered, $\mathbf{A}$ and $\mathbf{S}_q\mathbf{B}_q$ in $\mathbf{X} = \mathbf{A} \mathbf{B}_q^\top \mathbf{S}_q^\top$ are essentially unique by Lemma \ref{lemma:Identifiability}. 

Given that $\mathbf{S}_q$ is of full-column rank, $\mathbf{S}_q\mathbf{B}_q$ is a permutation with repetition of distinct rows in $\mathbf{B}_q$ and every row in $\mathbf{B}_q$ appears in $\mathbf{S}_q\mathbf{B}_q$. Regarding $\mathbf{S}_q\mathbf{B}_q$ as the subject of decomposition, $\mathbf{S}_q$ and $\mathbf{B}_q$ can be uniquely determined up to a permutation. Note that there is no diagonal scaling ambiguity between $\mathbf{S}_q$ and $\mathbf{B}_q$, since we have required $\mathbf{S}_q$ to be a matrix of $\{0,1\}$.

Overall, $\mathbf{A}$, $\mathbf{S}_q$ and $\mathbf{B}_q$ are essentially unique.
\end{proof}

\subsection{Proof of Proposition \ref{proposition:clusternumber}}
\label{appendix:proofofpropositionclusternumber}

\begin{proof}
Following the proof of Theorem \ref{theorem:CELIdentifiability}, we can similarly deduce that $\mathbf{A}'$ and $\mathbf{S}_{q+1}\mathbf{B}_{q+1}$ are essentially unique. Thus $\mathbf{S}_{q+1}\mathbf{B}_{q+1}=\mathbf{S}_{q}\mathbf{B}_{q}\mathbf{\Pi}\mathbf{D}$ where $\mathbf{D}$ is a full rank
diagonal nonnegative matrix and $\mathbf{\Pi}$ is a permutation matrix. From the assumption we know that there are only $M_q$ distinct rows in $\mathbf{S}_{q}\mathbf{B}_{q}$, thus $M_q$ distinct rows in $\mathbf{S}_{q+1}\mathbf{B}_{q+1}$. From the assumption that $\mathbf{S}_{q+1}$ being of full-column rank, we know that all rows of $\mathbf{B}_{q+1}$ appears in $\mathbf{S}_{q+1}\mathbf{B}_{q+1}$, which implies only $M_q$ distinct rows in $\mathbf{B}_{q+1}$, and they corresponds to a permutation/scaling of the rows in $\mathbf{B}_{q}$.
\end{proof}

\subsection{Uniqueness of the Optimal Cluster Number}
\label{appendix:UniquenessM_q}
\begin{proposition}
(Irreducibility of the Optimal Cluster Number) Following from the assumptions of Theorem \ref{theorem:CELIdentifiability}, there does not exist a decomposition of the data matrix such that $\mathbf{X} = \mathbf{A}' \mathbf{B}_{q-1}^\top \mathbf{S}_{q-1}^\top$, where $\mathbf{A}' \in\mathbb{R}^{N\times R}$, $\mathbf{B}_{q-1}\in \mathbb{R}^{M_{q-1}\times R}$, $\mathbf{S}_{q-1} \in \{0,1\}^{M\times M_{q-1}}$, $\|\mathbf{S}_{q-1}(j,:)\|_0 = 1,\forall j\in [1,M]$,  $rank(\mathbf{A}') = R$, $M_{q}>M_{q-1}\geq R$, and $\mathbf{S}_{q-1}$ is of full-column rank.
\label{proposition:UniquenessM_q}
\end{proposition}
\begin{proof}
Suppose such matrix decomposition exists. Then by Proposition \ref{proposition:clusternumber} we know that for the decomposition in Theorem \ref{theorem:CELIdentifiability}, $\mathbf{B}_{q}$ will have repeated rows. This will violate our assumption of $\mathbf{B}_{q}$ having no repeated rows, which proves the above proposition by contradiction. The above proposition shows that no $M_{q-1}<M_q$ satisfies Theorem \ref{theorem:CELIdentifiability}. 
\end{proof}

Together with Proposition \ref{proposition:clusternumber}, we know $M_q$ is unique, if the assumptions of Theorem \ref{theorem:CELIdentifiability} holds. Formally:

\begin{corollary}
(Uniqueness of the Optimal Cluster Number) Following from the assumptions of Theorem \ref{theorem:CELIdentifiability}, there does not exist a decomposition of the data matrix such that $\mathbf{X} = \mathbf{A}' \mathbf{B}_{p}^\top \mathbf{S}_{p}^\top$, where $\mathbf{A}' \in\mathbb{R}^{N\times R}$, $\mathbf{B}_{p}\in \mathbb{R}^{M_{p}\times R}$ has no repeated rows, $\mathbf{S}_{p} \in \{0,1\}^{M\times M_{p}}$, $\|\mathbf{S}_{p}(j,:)\|_0 = 1,\forall j\in [1,M]$, $rank(\mathbf{A}') = R$, $M_{p}\neq M_{q}$, and $\mathbf{S}_{p}$ is of full-column rank.
\end{corollary}

In summary, our Theorem \ref{theorem:CELIdentifiability} has identified a set of solutions of $\mathbf{A}$, $\mathbf{S}_q$ and $\mathbf{B}_q$, such that the solutions do  not contain repeated rows in $\mathbf{B}_q$ (otherwise the solution can be further reduced to a simpler form). Such a set of irreducible solutions naturally leads to an optimal cluster number. Also, the (essentially) uniqueness of the solutions implies the uniqueness of the optimal cluster number.


\section{More on Complexity of CEL(-Lite)}
\label{appendix:b time complexity}
\subsection{Time and Space Complexities of GPCA}
For CEL, time complexity of GPCA is $\mathcal{O}(\|\mathbf{S}_q(:,k)\|_0 N R)$, which is linear in $\|\mathbf{S}_q(:,k)\|_0$ and is the same as gradient calculation. 
PCA (including eigen decomposition) only involves an $R\times R$ matrix. Its time complexity is $\mathcal{O}(\|\mathbf{S}_q(:,k)\|_0 R^2+R^3)$, where $R$ is usually small. The space complexity is also low, which can be as low as $\mathcal{O}(R^2)$ in theory. Thus, split with GPCA is both time and space efficient. The time complexity can be further reduced via parallel computation. 

Note that when we use a constant number of (less or equal to) $n$ buffer interactions for each item considered in the GPCA, the time complexity will reduce to $\mathcal{O}(n\|\mathbf{S}_q(:,k)\|_0 R)$.

\subsection{Time Complexity of CEL-Lite Formally} 
\label{appendix:proof of online time complexity}

\begin{theorem}
The time complexity of CEL-Lite is $\mathcal{O}(DR)$.
\label{theorem:online time complexity}
\end{theorem}

\begin{proof}

The total time complexity is the sum of time complexity of embedding optimization and cluster optimization. The cluster optimization consists of split and reassignments.

Firstly, all reassignments can be processed within $\mathcal{O}(DR)$ time since it involves total $D/b$ batch; each batch incurs $\mathcal{O}(nmbR)$ time from $b$ sub-problems with $m$ candidate clusters and a constant number $n$ of buffered history considered.

Secondly, all the split operations can be processed within $\mathcal{O}(DR)$ time complexity. Here, we change $d$ in strategies \ref{strategy:1}\&\ref{strategy:2} to a monotonically non-decreasing sublinear function $d(D_t)$ to make it more general, where $D_t$ denotes the number of interactions in the current training step $t$, which is used to distinguish from $D$, the total number of interactions collected during the entire online learning. Then, strategies~\ref{strategy:1}\&\ref{strategy:2} ensure each newly split cluster in the current training step $t$ will have $\geq\frac{1}{2}d(D_t)$ interactions. This will remain true through the entire training, since the non-decreasing function $d$ ensures that such cluster will still have $\geq\frac{1}{2}d(D_{t'})\geq\frac{1}{2}d(D_{t})$ interactions even it splits in any future iteration $t'>t$. We now denote the number of newly split cluster in the current training step $t$ as $\Delta_t$ and suppose the training lasts for a total of $T$ steps.
Thus we can sum over them and obtain $\sum_{t=1}^{T} \Delta_t\times\frac{1}{2}d(D_t)\leq D$. 
Now, recall that the time complexity of each split operation is $\mathcal{O}(n\|\mathbf{S}_q(:,k)\|_0 R)$ when only a constant number of history considered. 
The time complexity for all splits in step $t$ then requires $\mathcal{O}\big((\Delta_t\times 2d(D_t)+\mathcal{O}(b))nR\big)$ time, where $\mathcal{O}(b)$ accounts for the new interactions from the reassignment or new batch data in step $t$.
As a result, we know that the overall time complexity for all splits are $\mathcal{O}\big(nR\sum_{t=1}^{T} (\Delta_t\times2d(D_t)+\mathcal{O}(b))\big)\leq\mathcal{O}\big(nR (4D+\mathcal{O}(D))\big)=\mathcal{O}(nDR)=\mathcal{O}(DR)$.\footnote{Note that strategy~\ref{strategy:1} incurs at most $\mathcal{O}(DR)$ time complexity additionally.}

Note that the embedding optimization also takes $\mathcal{O}(DR)$ time. So, the overall time complexity of CEL-Lite is $\mathcal{O}(DR)$. 
\end{proof}



\subsection{Complexity Analysis for CEL}
\label{appendix:offline time complexity}

In contrast to CEL-Lite, CEL assumes the data matrix is known a priori to the embedding learning. For CEL, in each step of optimization we will need to process the entire data matrix. 
With our framework in Section~\ref{sec:methodology}, CEL splits once per $t_2$ steps of embedding optimizations, until we have reached a desired cluster number $M^*$ (corresponding to a desired compression ratio). We perform the reassignment once per $t_1$ steps of embedding optimizations.
Suppose the training lasts for a total of $E$ steps (of embedding optimization). $E$ is a constant and maybe not negligible (or even very large) depending on the learning rate and the convergence speed.

\begin{theorem}
CEL with the maximum cluster number being fixed to $M^*$ has a time complexity of $\mathcal{O}\left(\left(E + M^* +\frac{E}{t_1}M^*\right) DR\right)$.
\end{theorem}

\begin{proof}

The embedding optimizations have a time complexity of $\mathcal{O}(DR)$ for each step. So the total time complexity of $E$ steps of embedding optimization is $\mathcal{O}(EDR)$. 

For the split operations, we can imagine a binary tree starting from a single root node. The children represent two clusters from the split of their parent, and let the value of a node represents the interactions (number of interactions associated with it). The time complexity will sum up all the nodes in the tree (except leaf nodes). We can easily verify that processing all nodes in a particular level (depth) of the tree sums to $\mathcal{O}(DR)$ time complexity if the level is completely filled (will be less if it is not completely filled), since the  numbers of interactions associated to the nodes sums to $D$. So now we will need to compute the depth of the tree. The tree is strictly a binary tree (but not necessarily full) with a maximum depth $M^*$. Therefore, the splits have a time complexity of $\mathcal{O}(M^* DR)$.

For the reassignment, the time complexity of each reassignment is $\mathcal{O}(M^* D R)$. We have totally $E/t_1$ reassignments, so we have a total time complexity of $\mathcal{O}(E/t_1 M^* D R)$ for reassignment. 

Thus, the overall time complexity is $\mathcal{O}\left(\left(E + M^* +\frac{E}{t_1}M^*\right) DR\right)$.
\end{proof}

\subsection{Practical Time Complexities}
Our theoretical analysis (Section~\ref{sec:online} and above) has shown CEL(-Lite) to be an efficient framework. Empirically, for example, for MovieLens-1M with DIN at the compression ratio of $5\%$ (Table~\ref{tab:movielens ranting prediction}), CEL runs for 1592s, while PreHash ($k=M_q$) takes 7256s. On large-scale datasets, Electronics with DIN at the compression ratio of $10\%$ (Table~\ref{tab:cr predictions}) as the example, CEL-Lite converges within 9366s. In comparison, PreHash ($k=M_q$) takes 31087s and the fastest baseline Modulo takes 6305s.

\section{More Details and Results}
\label{appendix: c the experiments}

\subsection{Initialization}
\label{appendix:Baseline Implementation Details}
All the embeddings (vectors) are randomly initialized: We first initialize the full embedding.
We sample a vector consisting of random Gaussian entries ($\mathcal{N}(0,1)$), then divided by the maximum absolute value to map each entry to a range of $[-1,1]$. For NMF, we take the absolute value of each entry to ensure the nonnegativity. Note that after each embedding optimization, we also apply the absolute function. Secondly, we initialize the cluster assignment such that users/items are assigned to clusters randomly. Then the cluster embeddings are initialized as the center of its associated users/items.

\subsection{Baseline Implementation Details}

\subsubsection{eTree and JNKM implementation details.} For eTree, we adopt a tree structure of depth $3$. The first layer is a single root node, the second layer consists of $M_q$ nodes, while the leaf nodes on the last layer represent the items. For JNKM, there are a total of $M_q$ clusters. 
During test prediction, we only keep the embeddings of second layer nodes of eTree, or the cluster embeddings of JNKM.

\subsubsection{BH and AE implementation details.} BH requires setting a hyperparameter $B$ (integer) that indicates the number of binary code blocks. Note that the total length of the binary coding $L$ should satisfy $2^L\geq M$, while the total embedding size is $B\times2^{L/B}$ should be close to $M_q$. We always choose an appropriate $B$ so that its memory usage $B\times2^{L/B}\geq M_q$ is just slightly larger than the embedding size $M_q$ required by the corresponding compression ratio in each experiment. The rest of the settings follow the original paper \cite{binaryhash}. For AE, we allow the top $M_q/2$ warmest items to have their own embedding, and the rest of the items are assigned to share the other $M_q/2$ embeddings through the Modulo hash method. 

\subsubsection{PreHash implementation details.} PreHash requires setting a hyperparameter $K$ to select the top-$K$ bucket for each user/item. In our experiments, we set $K$ to range from $4$ to $M_q$. The number of buckets $M_q$ is determined by the corresponding compression ratio in each experiment. For the rating prediction with NMF (Table~\ref{tab:movielens ranting prediction}), we do not use the historical interactions of users while hashing. Instead, we learn the hash assignment (weights associated with each bucket) during training. This is to ensure a fair comparison, since the historical interactions are also not used by the other baselines in the comparison. 
For the online conversion rate prediction (Table~\ref{tab:cr predictions}), we follow the original paper and use the user history to calculate the hash assignment. We collect and update the user history after receiving the streamed online data. The buckets are initialized randomly and then bind to the anchor users identified throughout online learning. The same operation is applied to items to reduce the item embedding size accordingly. The rest of the settings follow the original paper \cite{prehash}.

\begin{figure*}[t]
    \centering
    \includegraphics[height=30mm,trim={0 -25mm 0 0},clip]{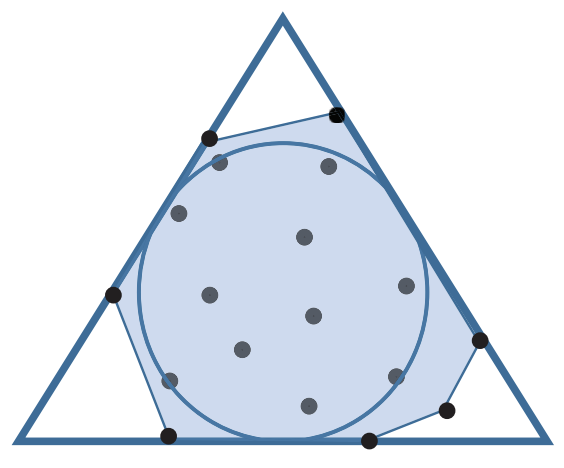}
    \hspace{40pt}
    \includegraphics[height=30mm]{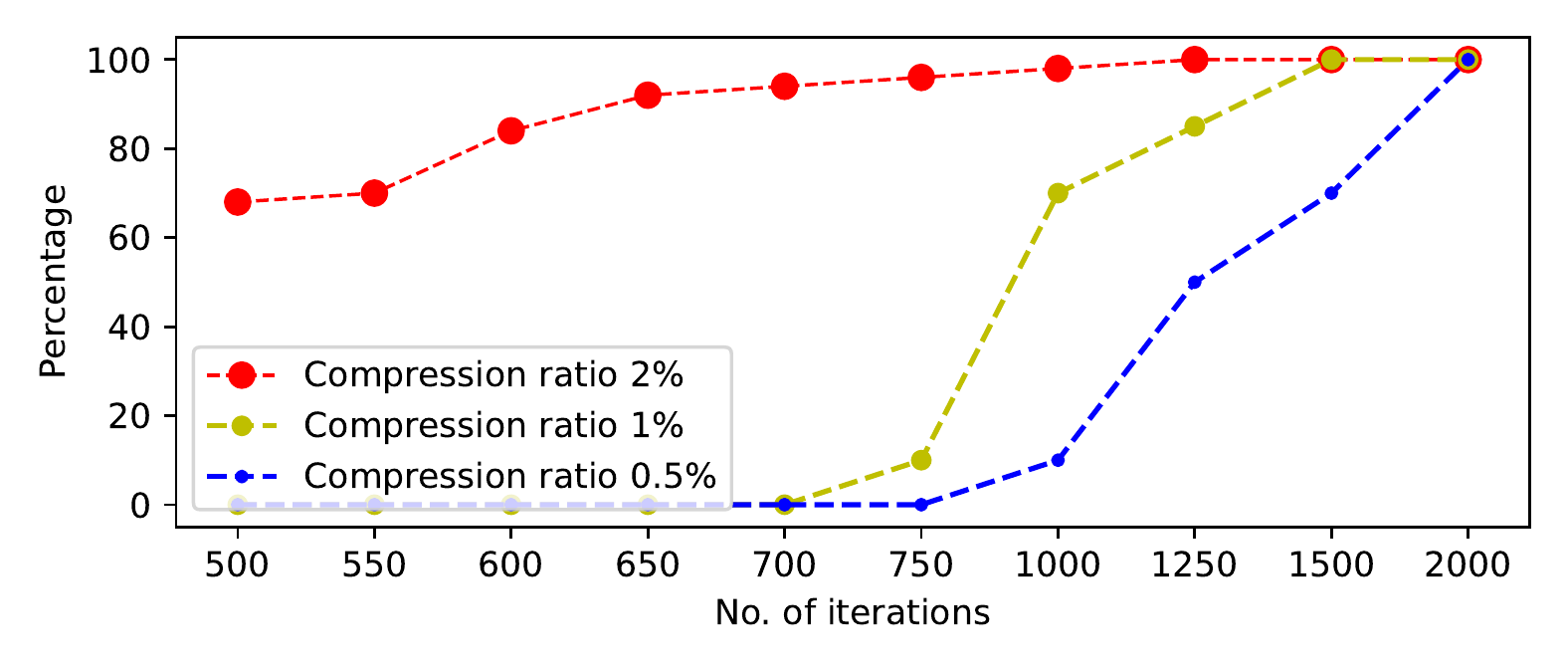}
    
    \hspace{-20mm}(a)\hspace{65mm}(b)
    \caption{(a)  Illustration of the sufficiently scattered condition. The dots are rows of $\mathbf{B}_q$; the triangle is the nonnegative orthant; the circle is the central region $\mathcal{C}$; the shaded region is $cone\{\mathbf{B}_q^\top\}$. Image is taken from a prior work~\cite{identifiability}. (b) The percentage of cluster embeddings (rows of $\mathbf{B}_q$) that lie outside the central region $\mathcal{C}$. We can see that in all settings, eventually all cluster embeddings will lie outside the central region $\mathcal{C}$.} 
    \label{fig:sufficiently scattered}
\end{figure*}

\subsection{Gradient Averaging and Optimizer}
\label{appendix:optimizer}

We optimize $\mathcal{L}$ with respect to $\mathbf{B}_q$ and $\mathbf{A}$ via first order optimizers, e.g., (plain) gradient descent (GD) and Adam. Since the number of assigned items can be significantly different among clusters, there may exhibit significant differences in gradient scales. 
To synchronize gradient descent training among different clusters, we divide the loss of each cluster by the number of associated items to switch from the total loss to the averaged loss:
\begin{equation}
    \Delta \mathbf{B}_q(k,:)=- \frac{\eta}{\|\mathbf{S}_q(:,k)\|_0} \;\;\partial\mathcal{L}_k/\partial \mathbf{B}_q(k,:)
\end{equation}
where $\eta$ is the learning rate, and the divisor is equal to the number of items assigned to the $k^{th}$ cluster. We refer to it as the \emph{gradient averaging}. It ensures that embeddings of all clusters are updated at a similar pace, avoiding larger clusters to be updated much faster.

We performed ablation studies on the optimizer used for gradient descent and the effect of gradient averaging. We found gradient averaging significantly reduces overfitting and keeps training stable. 

Table \ref{tab:gradaveragingvsno} and \ref{tab:optimizers} present the results of CEL with NMF on the MovieLens-100k dataset.\footnote{The trend for other feature interaction models and on other datasets are similar.} The results show that gradient averaging largely improves the training result of GD, while such improvement is less obvious for Adam. It could be because Adam, to some extent, has invariance towards the gradient scale. Among them, GD + gradient averaging achieves the best result overall.

Interestingly, without gradient averaging, we observe that the test MSE will start to go up after $\sim1000$ steps of embedding optimization, while training MSE will keep going down. It is possibly: Without gradient averaging, the embedding of clusters with limited assignments (and associated data) will be updated at the same speed as the big clusters. Such fast update of embeddings for small clusters may \emph{overfit} it to its limited interactions.

\begin{table}[htb]
\centering
\caption{The MSE results of different optimization settings for rating prediction tasks on MovieLens-100k. The item embedding table is compressed.} 
\scalebox{0.9}{
\begin{tabular}{c|cccc}
\hline
\makecell{Compression \\ Ratio} & \makecell{GD w. \\ grad avr} & \makecell{GD wo. \\ grad avr} & \makecell{Adam w. \\ grad avr} & \makecell{Adam wo. \\ grad avr} \\ \hline
$1\%$ & \textbf{0.9022} & 1.023 & 0.9139 & 0.9189 \\
$2\%$ & \textbf{0.9012} & 1.019 & 0.9150 & 0.9280 \\
\hline
\end{tabular}}
\label{tab:gradaveragingvsno}
\end{table}

\begin{table}[h]
\centering
\caption{The MSE results of different optimizers for rating prediction tasks on MovieLens-100k. The item embedding table is compressed to $2\%$.} 
\scalebox{0.9}{
\begin{tabular}{c|c|c|c}
\hline
GD & Adam & AdaGrad & RMSProp \\ \hline
\textbf{0.8654} & 0.8720 & 0.9160 & 0.8754     \\ \hline
\end{tabular}}
\label{tab:optimizers}
\end{table}

\subsection{Norm Regularization}
\label{appendix:scatteredness}

To handle the scaling arbitrary between $\mathbf{A}$ and $\mathbf{B}$ in NMF, we adopt a common solution by adding norm regularizations of $\mathbf{A}$ and $\mathbf{B}_q$ to the original objective function:
\begin{equation}
\label{equ:eonorm}
\textstyle
\mathcal{L}_{reg}=\frac{\lambda_{reg}}{2}\big(\|\mathbf{A}\|_{F}^2+\|\mathbf{B}_q\|_{F}^2\big).
\end{equation}
We set $\lambda_{reg}=1$ for our experiments. 

\subsection{The Sufficiently Scattered Condition}

As introduced in Appendix~\ref{appendix:all about Identifiability}, the sufficiently scattered condition being satisfied essentially means that conic hull of rows of $\mathbf{B}_q$ contains a central region $\mathcal{C}$ of the nonnegative orthant as its subset, illustrated in Figure~\ref{fig:sufficiently scattered}(a). 
Although it is hard to visualize or verify whether the sufficiently scattered condition is satisfied, we can easily check whether a cluster embedding (a row of $\mathbf{B}_q$) lies outside the central region $\mathcal{C}$. We thus perform such studies in Figure~\ref{fig:sufficiently scattered}(b). We can see that in all settings, eventually all cluster embeddings will lie outside the central region $\mathcal{C}$, implying that the embeddings (rows of $\mathbf{B}_q$) tend to satisfy the sufficiently scattered condition. Moreover, a larger compression ratio (which leads to a finer clustering structure) is more likely to get the embeddings to lie outside $\mathcal{C}$. 
Although there is no explicit regularization that encourages the clusters to be scattered, the norm penalty we applied in equation~\eqref{equ:eonorm} tends to push the embeddings to the boundary of the nonnegative orthant, which is a possible cause that we eventually learn a spread set of clusters.

\end{document}